\documentclass[nohyperref]{article}

\usepackage{microtype}
\usepackage{graphicx}
\usepackage{subfigure}
\usepackage{caption}
\usepackage{wrapfig}
\usepackage{diagbox}
\usepackage{booktabs} % for professional tables

\usepackage{hyperref}

\usepackage[accepted]{icml2022}

\usepackage{amsmath}
\usepackage{amssymb}
\usepackage{mathtools}
\usepackage{amsthm}
\usepackage{mathabx}
\usepackage[T1]{fontenc}

\usepackage[capitalize,noabbrev]{cleveref}

\theoremstyle{plain}
\newtheorem{theorem}{Theorem}[section]

\newtheorem{lemma}[theorem]{Lemma}

\theoremstyle{definition}
\newtheorem{definition}[theorem]{Definition}

\theoremstyle{remark}

\usepackage[textsize=tiny]{todonotes}

\usepackage{amsmath,amsfonts,bm,physics}

\def\eqref#1{equation~\ref{#1}}
\def\1{\bm{1}}

\DeclareMathAlphabet{\mathsfit}{\encodingdefault}{\sfdefault}{m}{sl}
\SetMathAlphabet{\mathsfit}{bold}{\encodingdefault}{\sfdefault}{bx}{n}

\def\gA{{\mathcal{A}}}

\def\gH{{\mathcal{H}}}

\def\gM{{\mathcal{M}}}

\def\gO{{\mathcal{O}}}

\def\gS{{\mathcal{S}}}

\def\sN{{\mathbb{N}}}

\def\sP{{\mathbb{P}}}

\def\sR{{\mathbb{R}}}

\newcommand{\E}{\mathbb{E}}

\newcommand{\hide}[1]{}
\newcommand{\yw}[1]{\textcolor{red}{[Yi: #1]}}

\icmltitlerunning{Revisiting Cooperative MARL Common Practices}

\begin{document}

\twocolumn[
\icmltitle{Revisiting Some Common Practices in \\Cooperative Multi-Agent Reinforcement Learning}



\icmlsetsymbol{equal}{*}

\begin{icmlauthorlist}
\icmlauthor{Wei Fu}{iiis}
\icmlauthor{Chao Yu}{tsee}
\icmlauthor{Zelai Xu}{tsee}
\icmlauthor{Jiaqi Yang}{ucb}
\icmlauthor{Yi Wu}{iiis,sqz}
\end{icmlauthorlist}

%\icmlaffiliation{yyy}{Department of XXX, University of YYY, Location, Country}
\icmlaffiliation{iiis}{Institute for Interdisciplinary Information Sciences, Tsinghua University, China}
\icmlaffiliation{ucb}{Department of Electrical Engineering and Computer Sciences, University of California, Berkeley, CA, USA}
\icmlaffiliation{sqz}{Shanghai Qi Zhi Institute, China}
\icmlaffiliation{tsee}{Department of Electronics Engineering, Tsinghua University, China}
%\icmlaffiliation{sch}{School of ZZZ, Institute of WWW, Location, Country}

\icmlcorrespondingauthor{Wei Fu}{fuwth17@gmail.com}
\icmlcorrespondingauthor{Yi Wu}{jxwuyi@gmail.com}

\icmlkeywords{Machine Learning, ICML}

\vskip 0.3in
]



\printAffiliationsAndNotice{} % otherwise use the standard text.

%%%%%%%%%%%%%%%%%%%%%%%%%%%%%%%%%%%
% major claim
%%%%%%%%%%%%%%%%%%%%%%%%%%%%%%%%%%%
% 1. common practice (i.e., value decomposed MARL) 在一些情况下有问题；但是PG没啥问题
% 2. individual PG 在不考虑数据和计算的情况下比policy sharing要好
% 3. suggestion: agent-id condition; auto-regressive
%%%%%%%%%%%%%%%%%%%%%%%%%%%%%%%%%%%
%%%%%%%%%%%%%%%%%%%%%%%%%%%%%%%%%%%
% paper structure
%%%%%%%%%%%%%%%%%%%%%%%%%%%%%%%%%%%
% 0. preliminary: summarize common practices (Q-decompose; V-decompose; Advantage-decompose)
% 1. xor game
%   (1) value decompose methods不行
%   (2) PG 可以学到optimal
%   ???(3) individual PG比policy sharing PG 好
%   > (4) empirical comparisons on xor game
% 2. n-player xor game
%   ???(1) <1.2><1.3>的结论能拓展到n-player xor game 么？
%   (2) a possible solution: auto-regressive pg
% 3. general markov games
%   (1) practical suggestions
%       (a) share policy + id (???? how to do agent-number generalization????)
%       (b) practical suggestions for auto-regressive
%       (c) pseudo-code
%   (2) grid world game example
% 4. more complex games
%   (1) mujoco
%   (2) football / SMAC
% 5. discussion and related work

\begin{abstract}
Many advances in cooperative multi-agent reinforcement learning (MARL) are based on two common design principles: value decomposition and parameter sharing. A typical MARL algorithm of this fashion decomposes a centralized Q-function into local Q-networks with parameters shared across agents.  Such an algorithmic paradigm enables centralized training and decentralized execution (CTDE) and leads to efficient learning in practice. Despite all the advantages, we revisit these two principles and show that in certain scenarios, e.g., environments with a  highly multi-modal reward landscape, value decomposition, and parameter sharing can be problematic and lead to undesired outcomes. In contrast, policy gradient (PG) methods with individual policies provably converge to an optimal solution in these cases, which partially supports some recent empirical observations that PG can be effective in many MARL testbeds. Inspired by our theoretical analysis, we present practical suggestions on implementing multi-agent PG algorithms for either high rewards or diverse emergent behaviors and empirically validate our findings on a variety of domains, ranging from the simplified matrix and grid-world games to complex benchmarks such as StarCraft Multi-Agent Challenge and Google Research Football. We hope our insights could benefit the community towards developing more general and more powerful MARL algorithms. Check our project website at \begin{small}{\url{https://sites.google.com/view/revisiting-marl}}\end{small}.
\end{abstract}

\section{Introduction}

Value decomposition has become the most popular paradigm for tackling cooperative multi-agent reinforcement learning (MARL) problems~\cite{sunehag2017value,wang2020off,rashid2018qmix,wang2020qplex,hu2021riit}. Under this paradigm, algorithms decompose the global Q-function into local Q-functions that satisfy the Individual Global Max (IGM) principle~\cite{son2019qtran}. Then the optimal global policy can simply be derived via greedily selecting locally optimal actions for each agent. This process enables centralized training and decentralized execution (CTDE)~\cite{ctde1,ctde2}, and the factorized representation has also led to enormous benefits, including credit assignment~\cite{zhang2020collaq}, emergent communication~\cite{wang2019learning}, exploration~\cite{liu2021cooperative}, etc.
% \yw{change} 
In addition, the local Q-networks are typically implemented with shared parameters leading to a shared replay buffer across agents, which not only reduces the number of total parameters but also results in more stable training~\cite{christianos2020shared}.
% \yw{check the claim and reference}

Despite all the aforementioned advantages, in this paper, we revisit popular cooperative MARL methods and argue that some common practices including value-decomposition representation and parameter sharing, can be problematic in certain scenarios. 
We start our discussion by considering a simple 2-by-2 matrix game, the \emph{XOR game}, which has two symmetric global optimal solutions. We prove that even for this particularly simple game, a value-decomposition-based algorithm cannot represent the underlying payoff structure, and thus may not learn the optimal policy via factorized local Q-networks. % with either shared or unshared parameters.
Interestingly, we notice that in this example, policy gradient (PG) with individual policies provably converges to either of the two optimal solutions. 
In addition, we also introduce another policy-based \emph{auto-regressive} representation and show that PG with such a representation can learn \emph{all} the strategy modes even on the $N$-player version of the XOR game, which has exponentially many optimal solutions. 
Our theoretical findings suggest that on certain multi-modal problems, policy gradient, which is often considered sample-inefficient as an on-policy method, can be preferable compared to popular value-based learning methods. This insight can be served as initial evidence to partially explain the recent empirical observations that PG methods perform surprisingly well on many MARL benchmarks~\cite{papoudakis2020benchmarking,yu2021surprising}.

We further extend our study to a grid-world game, \emph{Bridge}, which is a temporal variant of the XOR game.
Based on our theoretical insights, we present two practical suggestions for applying PG to cooperative Markov games concerning different evaluation targets.
For the highest final rewards and the fastest convergence, it can be beneficial to adopt an individual or agent-specific policy.
For a policy that can capture multi-modal behaviors, we propose an attention-based auto-regressive policy representation that has a minimal computation overhead while maximally retaining the expressiveness power of a joint policy. A few training techniques for efficiently learning multi-modal auto-regressive policies are also introduced in the Bridge game. % In addition, we also introduce a few training techniques on the Bridge game.
% while other baseline methods either fail or consistently converge to a single strategy mode. 
%\yw{check}

We also validate our two suggestions in more complex domains including the StarCraft Multi-Agent Challenge (SMAC)~\cite{starcraft} and Google Research Football (GRF)~\cite{kurach2019google}.
We show that in many scenarios, using individual or agent-specific policies can further improve the performance of the state-of-the-art multi-agent proximal policy optimization (PPO) method~\cite{yu2021surprising}. %\yw{check}
% the state-of-the-art multi-agent proximal policy optimization (PPO) implementation~\cite{yu2021surprising}.
Meanwhile, using our auto-regressive policy learning method, we can discover interesting emergent behaviors that are never discovered by existing fully decentralized PG or value-based methods.

We emphasize that our work is \emph{NOT} claiming for a new algorithm that is universally applicable. Instead, \textbf{we attempt to provide concrete analysis to explain recent empirical evidence on PPO's effectiveness and present practical alternatives of interest in certain scenarios.} 
We hope our theoretical and empirical findings can bring useful insights to the community towards developing more general and more powerful MARL algorithms in the future.

% Multi-agent (deep) reinforcement learning (MARL) is becoming the canonical method to solve various? problems [cite]...

% Centralized Training, Decentralized Execution (CTDE) is a popular paradigm of MARL. It trains ... It assumes individual max global... and limits the range of policies. The assumption may not hold for many important applications where multiple Nash equilibria (NE) exist. For example, ...

% To find better policies in environments with multiple NE, we deviate from this paradigm. We come up with a novel paradigm, Centralized Training, Sequential Execution (CTSE). We take advantage of the representation power of NN-based policies. ... 
% Our novel paradigm effectively solves the previously intractable environments. ... Besides, 

\section{Related Work}\label{sec:related}
%\yw{1. Value-decomposition algorithms for cooperative MARL; policy sharing}

Centralized Training with Decentralized Execution (CTDE)~\cite{lowe2017multi} is perhaps the most popular framework for MARL. The fundamental idea is to adopt global information for simplified training (i.e., by reducing a POMDP to an MDP) while maintaining policies that only take local information for producing actions. 
Value decomposition (VD) methods perform centralized Q-learning and represent the global Q-function as a combination of local Q-networks following the Individual-Global-Max (IGM) principle~\cite{son2019qtran}, i.e., the optimal joint action should be equivalent to the collection of greedy local actions of each agent, which naturally enables CTDE. 
Representative VD methods, such as Value Decomposition Network (VDN)~\cite{sunehag2017value}, QMIX~\cite{rashid2018qmix},
% QTRAN~\cite{son2019qtran}, 
QPLEX~\cite{wang2020qplex}, and other variants~\cite{rashid2020weighted,yang2020qatten}, have been providing increasing expressiveness power towards the function class satisfying the IGM principle.
Other works also introduce policy networks into the VD framework~\cite{su2020value,wang2020off,zhang2021fop,wu2021coordinated} or further adapt the decomposed Q-networks towards communication learning~\cite{wang2019learning}, zero-shot adaptation~\cite{hu2020updet}, credit assignment~\cite{zhang2020collaq}, or structural emergent behavior~\cite{wang2019influence,wang2020rode}.
Despite the success of the IGM principle, we argue in this paper that in certain scenarios with a highly symmetric multi-modal reward structure, VD methods or even the IGM principle itself can be problematic and therefore suggest feasible alternative directions for future MARL research. %Therefore, we suggest new potential directions for future MARL algorithmic research.

%\yw{2. policy gradient methods in MARL; some recent findings of PPO performance}

Another line of MARL methods, such as COMA~\cite{foerster2018counterfactual},
% MADDPG~\cite{lowe2017multi}, Actor-Attention-Critic~\cite{iqbal2019actor}, and HAPPO/HATRPO~\cite{kuba2021trust} \yw{are these really on-policy? need to double-check. MADDPG is not.},
adopt policy gradient (PG) and follow CTDE by learning local policies with a centralized value function as the critic. Due to the on-policy fashion, PG methods are typically believed to be less sample efficient and therefore less utilized in the academic literature with limited computation resources. Whereas, some recent empirical studies~\cite{de2020independent,papoudakis2020benchmarking,yu2021surprising} demonstrate that with proper input representation and hyper-parameter tuning, multi-agent PPO can achieve surprisingly strong performance and sample efficiency in many cooperative MARL benchmarks compared to off-policy VD methods.
Our theoretical analysis partially justifies these recent empirical findings and provides insights toward more effective implementation of PG methods. 
Similar theoretical conclusions can be also implied from a concurrent work~\cite{leonardos2022global}, which studies the convergence of PG in Markov potential games.
%\fuwei{A concurrent work~\citet{leonardos2022global} that studies the convergence of PG in Markov Potential Games, we study PG under environments with a highly multi-modal reward landscape.}

%\yw{3. multimodality; diversity behavior; correlated equilibrium; some recent theoretical works that fundamentally assume a single NE.}

Multi-modality is a common phenomenon in many multi-agent games. For example, many works have shown that diverse emergent behaviors can be obtained under a simple reward function~\cite{lowe2019pitfalls,baker2019emergent,tang2021discovering}. There is also a  direction of research focusing on discovering diverse behaviors, e.g., by evolution methods~\cite{cully2015robots,Pugh2016QualityDA}, population-based training~\cite{vinyals2019grandmaster,ParkerHolder2020EffectiveDI,lupu2021trajectory} or iterative policy optimization~\cite{Lanctot2017AUG,masood2019diversityInducing,zahavy2021discoveringDiverseNearlyOpt,zhou2022continuous}. In game theory, the multi-modality issue is closely related to the concept of equilibrium refinement~\cite{kreps1982sequential}, which studies selecting the desired Nash equilibrium. 
The XOR game can be also interpreted as a cooperative version of the chicken game, which corresponds to the concept of correlated equilibrium~\cite{aumann1974subjectivity}. Our work points out that the existence of multiple optimal strategies can be an issue for existing MARL methods. We also develop an auto-regressive representation for representing a single multi-modal policy.
% \textbf{[2] assumes all different optimal solutions are known} in advance and focuses on making a consensus on solution selection among agents, while we focus on \textbf{policy learning from scratch}, which is a much harder challenge.  Moreover, PG-AR aims to \textbf{directly learn a multi-modal policy}, which is itself a novel objective in the MARL literature
Existing MARL literature primarily applies sequential execution for learning a consensus solution among agents by assuming all optimal solutions are known in advance~\cite{DBLP:conf/tark/Boutilier96}.
% In contrast to the existing MARL literature
% which primarily applies sequential execution to make
% on sequential execution~\cite{DBLP:conf/tark/Boutilier96},
% which primarily considers refining known Nash equilibria,
By contrast,
we focus on directly learning a multi-modal policy from scratch, which is more challenging.
Developing more powerful algorithms for diverse behaviors or sophisticated equilibria is beyond the scope of this paper.

We remark that the auto-regressive representation is widely adopted in the language generation literature, from N-gram models~\cite{damerau2018linguishtic} to recent attention-based architectures~\cite{vaswani2017attention,devlin2018bert}, which is the foundation of our proposed auto-regressive policy learning. 
There is a recent trend in developing non-auto-regressive text generation methods~\cite{gu2017non,qian2020glancing} for fast convergence and computational efficiency, which conceptually motivates our proposal of individual policy learning.
In recent MARL literature,
sequential policy update has been proposed to guarantee monotonic policy
improvement~\cite{DBLP:journals/corr/abs-2109-11251}, which improves the performance of PG methods under CTDE.
Instead of focusing on efficient training for decentralized execution, our proposed auto-regressive representation improves the policy modeling capacity and maintains the full expressiveness power of a centralized policy. Hence, we can learn diverse behaviors using a single auto-regressive policy.
% centralized learning approximation for diverse behaviors
% has a parallel focus on improving representation instead of training, while these techniques can be effortlessly combined.
% A concurrent work~\cite{DBLP:journals/corr/abs-2109-11251} proposes a cooperative MARL algorithm to auto-regressively update individual policies, which guarantees monotonic joint policy improvement. 
% However, it provides a potential way to further improve our auto-regressive policy learning. We present preliminary results in the appendix and leave further discussion to future works.
Finally, auto-regressive policy learning requires broadcasting each agent's action to all the following agents, which assumes a perfect communication channel on actions and is therefore related to multi-agent communication~\cite{foerster2016learning,wang2019learning}.

\section{Background}

% \yw{update the notation: $\mathcal{A}\to\mathcal{A}^n$; remove horizon; subscript/superscript}

\subsection{Notation}
We study cooperative MARL under the framework of multi-agent Partially Observable Markov Decision Process~\cite{DBLP:conf/tark/Boutilier96,DBLP:journals/ai/KaelblingLC98},
% Decentralized Partially-Observable Markov Decision Process (Dec-POMDP) \citep{oliehoek2008optimal}, 
which is defined as a tuple $\gM = \langle n, \gS, \gA, \gO, r, P, O, \gamma, H\rangle$. Here, $n \in \sN$ is the number of agents, $\gS$ is the state space, $\gA$ and $\gO$ are the action and observation space of each agent, $r : \gS \times \gA^n \to \sR$ is the reward function, $P$ is the transition model, $O : \gS \to \gO$ is the observation function, $\gamma$ is the discount factor, and $H$ is the horizon.
For state $s, s' \in \gS$ and a joint action $\mathbf{a} \in \gA^n$, the transition probability of reaching state $s'$ from state $s$ by executing action $\mathbf{a}$ is $P(s' \mid s, \mathbf{a})$.
At timestep $t$, $s_t$ denotes the state and each agent $i\in[n]$ receives a observation $o_t^i=O(i, s_t)$. Then, given the joint observation $\mathbf{o}_t=(o_t^1,o_t^2,\dots,o_t^n)$, agents output a joint action $\mathbf{a}\in\gA^n$ according to the joint policy $\pi : \gO^n \to \triangle\left(\gA^n\right)$. 
% We will assume that the joint action space satisfies $\gA = \gA_1 \times \cdots \times \gA_n$, where $\gA_i$ is the action space of the $i$-th agent. In this paper, we consider homogeneous action spaces, i.e., $\gA_1=\gA_2=\dots=\gA_n$.
We aim to find the optimal joint policy $\pi^\star$ to maximize the expected return, i.e.,
\[
\pi^\star =\arg\max_\pi \mathbb{E}_{(s_t,\mathbf{a}_t)\sim(P,\pi)}\left[\sum_{t=1}^H \gamma^{t-1}r(s_t,\mathbf{a}_t)\right].
\]
%where $\gamma\in (0,1)$ is a discount factor.

% \yw{The definition is problematic. You are actually finding a policy from observation $\mathcal{O}$ to action rather than from state $\mathcal{S}$ to action. Correct this.}

% \yw{1. notation}

% \yw{MDP definition; Learning objectives.}

\subsection{Value Decomposition in MARL}

% \yw{The logic should be like this: (1) explain VD methods are essentially centralized Q learning with $Q_tot$. (2) VD methods decompose centralized Q into local Q networks. (3) explain the IGM principle --- why this is a good principle. You should first explain the relationship between global Q and local Q before you can explain IGM.}

Value-based methods in cooperative MARL tasks aim to learn a global Q-function $Q_{\textrm{tot}} : \gS\times\gA^n \to \mathbb{R}$ to estimate the future expected return given current state $s$ and joint action $\mathbf{a}$. However, the global state $s$ is typically unavailable during execution and the dimension of joint action space $\mathcal{A}^n$ grows exponentially w.r.t. agent number $n$. %large action space $\gA$ significantly increases the difficulty of centralized training. 
Value decomposition (VD) addresses this issue by decomposing the global Q-value $Q_{\textrm{tot}}(s,\mathbf{a})$ into local Q-values $Q_i(o^i,a^i)$. In particular, popular VD methods, such as VDN~\cite{sunehag2017value} and QMIX~\cite{rashid2018qmix}, learn local Q-functions $Q_i : \gO\times\gA \to \mathbb{R}$ and a mixing function $f_{\textrm{mix}}(\cdot;s)$ conditioning on the state $s$ to represent the global Q-value $Q_{\textrm{tot}}(s,\mathbf{a})$ by
\begin{align}\label{eq:vd}
Q_{\textrm{tot}}(s,\mathbf{a})=f_{\textrm{mix}}(Q_1(o^1,a^1),\dots,Q_n(o^n,a^n);s),
\end{align}
% \fuwei{where $f_{\textrm{mix}}^\textrm{mono}(\cdot;s)$ satisfies the monotonicity constraint
% $\frac{\partial Q_{\textrm{tot}}}{\partial Q_i}\ge 0$. For stronger expressiveness under the IGM principle, QPLEX~\cite{wang2020qplex} decomposes the joint advantage $A_\textrm{tot}(s,\mathbf{a})$ as
% \begin{align}\label{eq:qplex-vd}
%     A_{\textrm{tot}}(s,\mathbf{a})&=f^\textrm{QPLEX}_{\textrm{mix}}(A_1(o^1,a^1),\dots,A_n(o^n,a^n);s,\mathbf{a}),
% \end{align}
% where $f^\textrm{QPLEX}_{\textrm{mix}}(\cdot;s,\mathbf{a})$ additionally conditions on the joint action $\mathbf{a}$
% and can be a non-monotonic function.
% }
% \yw{directly make $f_{mix}$ depend on $s$.  Don't mix up algorithms and implemenation. Also use textrm, i.e., $f_{\textrm{mix}}$}
For a reduced model size and efficient training, the parameters for each local Q-function $Q_i$ are shared, which is called \emph{parameter sharing}.
$f_{\textrm{mix}}$ is enforced to satisfy the \emph{Individual-Global-Max} (IGM)~\cite{son2019qtran} principle such that
\emph{any} optimal joint action $\mathbf{a}^*$ satisfies
\begin{align}\label{eq:igm}
\mathbf{a}^*=\arg\max_{\mathbf{a}\in\gA^n}Q_{\textrm{tot}}(s,\mathbf{a})=\cup_{i=1}^n\{\arg\max_{a^i\in\gA}Q_i(o^i,a^i)\}.
\end{align}
Therefore, all optimal joint actions can be easily derived by independently choosing a local optimal action from each local Q-function $Q_i$, which enables centralized training and decentralized execution (CTDE).
% Different VD methods propose different representations for the mixing function with various expressiveness capabilities.

Notably, as a more recent VD variant, QPLEX~\cite{wang2020qplex}, decomposes the global advantage function $A_{\textrm{tot}}(s,\mathbf{a})$, which is defined by
\begin{equation}\label{eq:ad}
A_{\textrm{tot}}(s,\mathbf{a})=Q_{\textrm{tot}}(s,\mathbf{a})-\max_\mathbf{a}Q_{\textrm{tot}}(s,\mathbf{a}),
\end{equation}
and learns a mixing function $f_{\textrm{mix}}^\textrm{A}(\cdot;s,\mathbf{a})$ to represent $A_{\textrm{tot}}(s,\mathbf{a})$ via local advantage functions, i.e.,
\begin{align}\label{eq:adv-vd}
A_{\textrm{tot}}(s,\mathbf{a})=f_{\textrm{mix}}^\textrm{A}(A_1(o^1,a^1),\dots,A_n(o^n,a^n);s,\mathbf{a}),
\end{align}
where $A_i(o^i,a^i)$ is the local advantage of agent $i$. 
QPLEX follows an \emph{advantage-based IGM principle}
\begin{align}\label{eq:adv-igm}
\mathbf{a}^*=\arg\max_{\mathbf{a}\in\gA^n}A_{\textrm{tot}}(s,\mathbf{a})=\cup_{i=1}^n\{\arg\max_{a^i\in\gA}A_i(o^i,a^i)\}.
\end{align}
It can be proved that QPLEX can represent the any functions satisfying the IGM principle.

\subsection{Policy Learning in MARL}
\label{sec:bg-policy}

% \yw{3. PG formulation}

Policy gradient (PG) and its variants~\cite{williams1992simple,schulman2017proximal,foerster2018counterfactual} directly optimize the policy $\pi$ by gradient descent.
In cooperative MARL tasks, popular PG methods follow CTDE by learning an individual actor  $\pi_{\theta_i}: \gO \to \triangle(\gA)$ for agent $i$ parameterized by $\theta_i$, and a centralized value function $V_{\psi_i} : \gS \to \mathbb{R}$ (i.e., critic) parameterized by $\psi_i$.
In order to leverage global information, the value function takes the global state $s$~\cite{yu2021surprising} or the combination of all the local observations $[o^1,\ldots,o^n]$~\cite{lowe2017multi} as its input for an accurate global value estimate. 
The majority of existing multi-agent PG works adopt parameter sharing, i.e., $\theta_1=\cdots=\theta_n=\theta$ and $\psi_1=\cdots=\psi_n=\psi$, for reducing model size.
Such a decentralized formulation naturally induces the following implicit joint policy with a \emph{fully independent} factorization: 
%While centralized PG training with $\pi_\theta$ is a direct extension from single-agent PG, it introduces severe dimensionality issue and significantly large variance. Therefore, decentralized policy learning is typically preferred in the existing literature, which factorizes $\pi_\theta$ into the form of
\begin{align}\label{eq:pg}
\pi(\mathbf{a}\mid \mathbf{o}) \approx \prod_{i=1}^n \pi_{\theta_i}(a^i\mid o^i).
\end{align}
% With such a policy representation, each agent $i$ simultaneously performs policy learning by running the single-agent policy gradient over its individual policy $\pi_{\theta_i}$ and assuming other policies unchanged. 
%\vspace{1cm}

\section{A Motivating Example: XOR Game}
\label{sec:xor}

% \yw{1. the XOR Game}

% \yw{2. Value decomposition cannot represent the XOR game; also cannot learn the optimal}

% \yw{3. Policy Sharing cannot learn the XOR game}

% \yw{4. policy gradient individual policy can learn the optimal solution}

% \yw{5. Auto-Regressive Policy; and show AR policy can learn multiple modes on $N$-player XOR game}

%The individual-global-max (IGM) principle is the fundamental assumption of the CTDE paradigm. 
%\begin{definition}[IGM principle, \citet{son2019qtran}] 
%\end{definition}

% \yw{Describe that we start our discussion from a simple matrix game setting, which does not contain states.}

%\subsection{XOR Game}
%\label{sec:xor}

% \yw{First describe the game. and then state theoretical results. }

% \begin{table}[t]
% 	\caption{XOR Game. \yw{update; wrap}}
% 	\label{tbl:xor}
% 	\vskip 0.15in
	
% 	\vskip -0.1in
% \end{table}
\begin{wraptable}{R}{.35\columnwidth}
  \centering
  \vspace{-2mm}
    \begin{center}
		\begin{small}
			\begin{sc}
				\begin{tabular}{|l|c|c|} \hline 
				%  \footnotesize{\diagbox{$a^2$}{$a^1$}}& 1 & 2 \\ \hline 
					 0 & 1 \\ \hline 
					1 & 0 \\ \hline 
				\end{tabular}
			\end{sc}
		\end{small}
	\end{center}
	\vspace{-3mm}
      \caption{\small{Payoff matrix of XOR game.}}
\label{tbl:xor}
\end{wraptable}
We start our analysis by considering a cooperative game in the simplest form of a 2-by-2 matrix game called \emph{XOR game} as shown in \cref{tbl:xor}. 
In this 1-step stateless game, each agent has two possible actions and they will get a positive reward only if they output different actions. There are two symmetric and equally optimal strategies in this game.
We will show that value-decomposition methods may fail to converge on this particularly simple game (Sec.~\ref{sec:xor-vd}) but policy gradient methods can be proved to converge (Sec.~\ref{sec:xor-pg}). Lastly, we will show a simple PG variant that can learn a policy covering all possible modes even in the $n$-player extension of the XOR game (Sec.~\ref{sec:xor-ar}).

%We show that even the simplest cooperative strategy cannot be expressed in a local form. Consider the XOR game with the following payoff function, as shown in \cref{tbl:xor}.

% \yw{@jiaqi: can we change this to the permutation version?}
% \jiaqi{Yes, let's do it.}

%Here, we define several popular approaches to MARL in Dec-POMDP, and we investigate and compare their powers of representing the optimal policies in MARL, under the example of the XOR game (\cref{def:xor}).

\subsection{Value Decomposition in XOR game}
\label{sec:xor-vd}

We assume each agent has two actions, 1 and 2. In this stateless setting, we use $Q_1$ and $Q_2$ to denote the local Q-functions, and $Q_\textrm{tot}$ to denote the global Q-function.
The IGM principle requires that for any optimal joint action $(a^{1*},a^{2*})$,
\begin{align}\label{eq:xor-igm}
(a^{1*},a^{2*})&=\arg\max_{a^1,a^2}Q_\textrm{tot}(a^1,a^2)\\
&=\{\arg\max_{a^1}Q_1(a^1),\arg\max_{a^2}Q_2(a^2)\}.
\end{align}
% the global Q-value degenerates to $Q_{\textrm{tot}}(a^1,a^2)$ and can be further decomposed as
% \begin{equation}\label{eq:xor-vd}
%     Q_{\textrm{tot}}(a^1,a^2)=f_{\textrm{mix}}(Q_1(a^1),Q_2(a^2)),
% \end{equation}
% where $f_{\textrm{mix}}$ needs to satisfy the IGM principle.

% Similarly,
% QPLEX applies the advantage-based IGM principle, and decomposes the global advantage to
% \begin{align}\label{eq:xor-qplex}
%     A_{\textrm{tot}}(a^1,a^2) &= f_{\textrm{mix}}^\textrm{QPLEX}(A_1(a^1),A_2(a^2)) \\
%     &= \sum_{i=1}^2 \lambda_i(a^1, a^2) A_i(a^i),
% \end{align}
% where $\lambda_i(a^1, a^2) > 0$.

\begin{theorem} Value decomposition (Eq.~(\ref{eq:xor-igm})) cannot represent the underlying global Q-function in XOR game. \label{thm:val}
\end{theorem}

\begin{proof}
% We prove this theorem by contradiction.
In XOR game, the desired global Q-value is the payoff. Specifically, if value decomposition methods were able to represent the payoff, we would have $Q_\textrm{tot}(1, 2)=Q_\textrm{tot}(2,1)=1$ and $Q_\textrm{tot}(1, 1)=Q_\textrm{tot}(2,2)=0$.

If either of these two agents has different local Q values, e.g. $Q_1(1)> Q_1(2)$, we have $\arg\max_{a^1}Q_1(a^1)=1$. Then \emph{any} optimal joint action
\begin{align}
(a^{1*},a^{2*})&=\arg\max_{a^1,a^2}Q_\textrm{tot}(a^1,a^2)\\
&=\{\arg\max_{a^1}Q_1(a^1),\arg\max_{a^2}Q_2(a^2)\} 
\end{align}
satisfies $a^{1*}=1$ and $a^{1*}\neq 2$ according to the IGM principle, which indicates that the joint action $(a^1,a^2)=(2,1)$ is sub-optimal, i.e., $Q_\textrm{tot}(2,1)<1$.
Otherwise, if $Q_1(1)=Q_1(2)$ and $Q_2(1)=Q_2(2)$, then $Q_\textrm{tot}(1, 1)=Q_\textrm{tot}(2,2)=Q_\textrm{tot}(1, 2)=Q_\textrm{tot}(2,1)$.

As a result, value decomposition cannot represent the payoff matrix of the 2-player permutation game.

% Let $(\alpha_1, \alpha_2, \beta_1, \beta_2) = (Q_1(1), Q_1(2), Q_2(1), Q_2(2))$. Suppose for contradiction that it could represent. Then $f_{\textrm{mix}}(\alpha_1, \beta_2) = f_{\textrm{mix}}(\alpha_2, \beta_1) = 1$. Now if $\alpha_1 \ge \alpha_2$ then by monotonicity we have 
% \begin{align}
% 	0 = Q_{\mathrm{tot}}(1, 1) = f_{\textrm{mix}}(\alpha_1, \beta_1) \ge f_{\textrm{mix}}(\alpha_2, \beta_1) = 1,
% \end{align}
% contradiction. Otherwise, we have 
% \begin{align}
% 	0 = Q_{\mathrm{tot}}(2, 2) = f_{\textrm{mix}}(\alpha_2, \beta_2) \ge f_{\textrm{mix}}(\alpha_1, \beta_2) = 1,
% \end{align}
% contradiction.

% QPLEX uses advantage-based IGM, and decomposes the global advantage value as 
% \begin{align}\label{eq:xor-qplex}
%     A_{\textrm{tot}}(a^1,a^2) &= f_{\textrm{mix}}(A_1(a^1),A_2(a^2)) \\
%     &= \sum_{i=1}^2 \lambda_i(a^1, a^2) A_i(a^i),
% \end{align}
% where $\lambda_i(a^1, a^2) > 0$. 

% According to Eq.~(\ref{eq:xor-qplex}), since $f_{\textrm{mix}}^\textrm{QPLEX}$ is still monotonic, we only need to change the Q-value to advantage to see contradiction. Let $(\alpha_1, \alpha_2, \beta_1, \beta_2) = (A_1(1), A_1(2), A_2(1), A_2(2))$, and suppose QPLEX could represent the global advantage value. Then if $\alpha_1 \ge \alpha_2$, we have
% \begin{align}
% 	-1 = f_{\textrm{mix}}(\alpha_1, \beta_1) \ge f_{\textrm{mix}}(\alpha_2, \beta_1) = 0,
% \end{align}
% contradiction. Otherwise, we have 
% \begin{align}
% 	-1 = f_{\textrm{mix}}(\alpha_2, \beta_2) \ge f_{\textrm{mix}}(\alpha_1, \beta_2) = 0,
% \end{align}
% contradiction.

\end{proof}

\begin{wrapfigure}{r}{0.5\columnwidth}
\centering\vspace{-25pt}
\subfigure{\vspace{-4pt}\includegraphics[width=0.43\columnwidth]{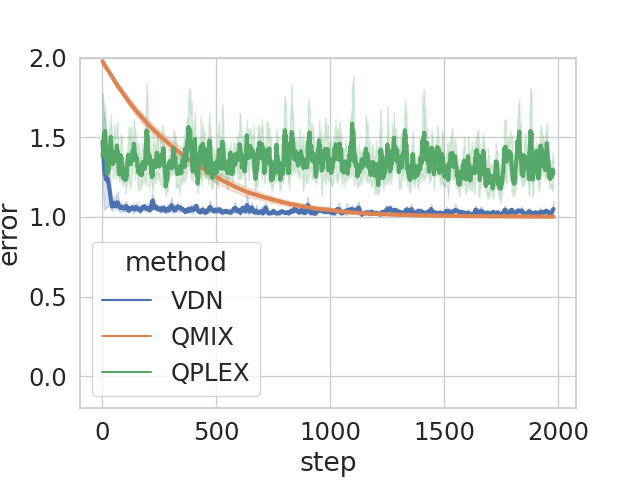}}
\vspace{-3mm}
\caption{Error in XOR game.}
\label{fig:xor-loss}
\vspace{-2mm}
\end{wrapfigure}
\textbf{Remark: }  \cref{thm:val} holds for \emph{all} VD methods that follow the IGM principle defined in Eq.~(\ref{eq:igm}). It suggests that \emph{global Q-values learned by VD for the XOR game can be arbitrary.
% Q-learning will NOT converge on XOR game
}
We conduct empirical experiments by applying state-of-the-art VD methods, including VDN~\cite{sunehag2017value}, QMIX~\cite{rashid2018qmix} and QPLEX~\cite{wang2020qplex}, by uniformly fitting the payoff matrix and average the results over 6 random seeds.
Note that since the XOR game is stateless, Q-values are just separate learnable constants.
Fig.~\ref{fig:xor-loss} shows the error between learned global Q-values and the payoff matrix throughout learning process.
None of the VD methods achieve a zero regression loss as pointed out by \cref{thm:val}.
% We can also observe that QPLEX suffers from significant loss fluctuations.
%VD-methods never achieve zero regression loss
% This implies that the eventual policy induced by the Q-functions learned by VD-based methods can be completely arbitrary due to the condition of the optimization process.
We additionally emphasize that Theorem~\ref{thm:val} holds whenever the local Q networks are shared or not --- the failure is due to the limited representation power of VD methods.
More discussions can be found in Appendix~\ref{app:xor-vd-qlearning}.
%This result also holds for advantage decomposition methods~\cite{wang2020qplex}, because the Q-value can be computed from the advantage function. , so if advantage decomposition can represent the optimal value function, then it would contradict our \cref{thm:val}.

\iffalse
We note that the same result holds for advantage decomposition methods~\cite{wang2020qplex}, because the Q value can be computed from the advantage function, so if advantage decomposition can represent the optimal value function, then it would contradict our \cref{thm:val}.
We show the loss of Q-learning across training the XOR game in Fig.~\ref{??}. Results are averaged over 6 random seeds. Q-learning never achieves zeros regression losses, which implies Thm~\ref{thm:val}.

\yw{add a corollary saying the Q-learning fails to converge to the optimal mode under this scenario}

\yw{TODO: add intuitive explanations}

\fi

% \yw{TODO: add a wrapped figure showing the simulation result of Q-iteration}

%\begin{theorem} Symmetric Policy cannot represent the optimal policy of the XOR game.
%\end{theorem}

%Clearly, an $n$-agent XOR game has $2^{n-1}$ optimal policies, but 

\subsection{Policy Gradient in XOR game}
\label{sec:xor-pg}

%\yw{TODO: add texts}

%\yw{TODO: specifically write down the definition w.r.t. the 2-player setting}

We refer \emph{shared policy learning} (PG-sh) to the setting of learning a single policy parameter $\theta$ for all the agents, i.e., $\pi_1=\pi_2=\pi_\theta$, and refer \emph{individual policy learning} (PG-Ind) to the setting of learning a separate policy parameter $\theta_i$ for each agent's policy $\pi_{\theta_i}$. We will show by the following theorems that a shared policy cannot solve the XOR game while individual policy learning can provably converge to an optimal solution. 

\iffalse
\begin{definition}[Shared policy] Shared policy methods learn a policy $\pi_1 = \cdots = \pi_n : \gS \to \triangle(\gA)$.
\end{definition}

\begin{definition}[Individual policy] Individual policy methods  learn a policy $\pi_i : \gS \to \triangle(\gA)$ for each agent. 
\end{definition}
\fi

\begin{theorem} Shared policy learning cannot learn an optimal policy for the XOR game. \label{thm:share-unsolvable}
\end{theorem}

\begin{proof} Let $\pi_i$ be a shared policy and $a^i$ denote the action by agent $i$. Let $\alpha = \sP(a^i = 0)$. The expected return of $\pi_i$ is 
	\begin{align}
		\E[R(\pi_i)] &= \sP(a^1 = 0, a^2 = 1) \notag \\
		&\qquad + \sP(a^1 = 1, a^2 = 0) \\ 
		&= 2\alpha(1-\alpha) \le 0.5 < 1,
	\end{align} 
but the optimal return is $1$. 
\end{proof}

\begin{lemma} Individual policies can represent the optimal policy for the XOR game. %$n$-agent permutation game.
\end{lemma}

\begin{proof}
%\fuwei{
In the XOR game, there exists a deterministic optimal (joint) policy.
Then, we can construct deterministic individual policies for each agent w.r.t. the global optimum.
%}
% For any optimal mode, we can construct a deterministic policy for each agent towards that mode.
\end{proof}

%\begin{proof} Note that an optimal policy can be represented by $\pi_i(a^i) = \ind\{a^i=x_i\}$, for each permutation $\{x_{1}, \cdots, x_n\}$ over 1 to $n$. Here $n=2$ \yw{omit the proof?}
%\end{proof}

%We note that in the permutation game, all local maxima are also global maxima. 
\begin{theorem} Individual policy learning via stochastic policy gradient can learn an optimal policy in the XOR game. %$n$-agent permutation game.
\label{thm:pg-ind-xor}
\end{theorem}

\begin{proof} %We assume sufficient exploration in this toy example. 
Previous works have shown that SGD can escape saddle points and converge to local optima under mild assumptions~\cite{kleinberg2018alternative}. % (e.g. the main assumption in \cite{kleinberg2018alternative}, which has been observed to hold for neural networks and real data mini-batches). %We further show that all local maxima are global.
%}
% Because policy gradient is based on SGD which converges to local maxima under mild assumptions~\citep{kleinberg2018alternative}, here we only need to show that all local maxima are global, which is because
For any individual policy $\pi_i$, if it has a positive reward but not the maximum reward (which is $1$), then we can find a better policy in the neighborhood by increasing probability along the direction of the permutation where $\pi_i$ puts maximum probability. Otherwise, if it has a zero reward, then we can find a better policy by increasing probability in all directions.
\end{proof}

\textbf{Remark: }
% \fuwei{\cref{thm:pg-ind-xor} is also supported by a concurrent work~\cite{leonardos2022global}, which shows the provable convergence of PG in the more general Markov Potential Games.}
Through the theoretical analysis in Sec.~\ref{sec:xor-vd} and~\ref{sec:xor-pg}, we show that even a simple 2-by-2 matrix game, i.e., XOR game, can be a counterexample that VD-based MARL algorithms fundamentally fail to converge. By contrast, PG methods, as an SGD-based approach, can provably converge to an optimum.
% We emphasize that the theorems themselves are not the main contributions of our paper. \textbf{We are presenting a ``trivial'' counter-example where SOTA MARL algorithms, i.e., VD methods, fundamentally fail to converge.} By contrast, PG methods, as an SGD approach, can provably converge to \textbf{AN} optimum.
This fact suggests that, even though PG methods are much less utilized in MARL compared with VD methods, it can be preferable in certain cases, e.g., games with multiple strategic modalities, which can be common in real-world applications. 

\subsection{Auto-Regressive Policy Learning}
\label{sec:xor-ar}

Although we have proved in the previous discussions that individual PG can effectively learn an optimal mode in the XOR game, the strategy mode that it finally reaches can be highly dependent on the initialization of the policies. 
Thus, a natural question will be:

%\begin{itemize}
    %\item 
    \emph{Can we learn a single policy that can cover all the optimal modes?}
%\end{itemize}

We remark that learning multi-modality policies is meaningful for a wide range of focuses, such as emergent behavior~\cite{tang2021discovering}, exploration~\cite{mahajan2019maven}, learning to adapt~\cite{Lanctot2017AUG,balduzzi2019open} or interacting with humans~\cite{hu2020other}. %, etc.
% \yw{cite: RPG; Maven; double oracle/meta-learning; FCP}

Let's re-visit a global policy $\pi(a^1,a^2)$, which models the joint action probabilities for both agents. It would be trivial to construct a multi-modal global policy such that it has equal chances to output joint actions of either $(2,1)$ or $(1,2)$. 
However, following the decentralized policy gradient formulation (Eq.~\ref{eq:pg}), the factorized representation $\pi(a^1,a^2)\gets\pi(a^1)\pi(a^2)$ is only able to represent one particular mode. Therefore, for a multi-modal policy, we need to develop a policy representation with a stronger expressiveness of the joint policy and with minimal computation overhead compared with independent policies (Eq.~\ref{eq:pg}).

We propose to represent the policy in an \emph{auto-regressive} form, i.e., $\pi(a^1,a^2)=\pi(a^1)\pi(a^2|a^1)$ in XOR game. 

Formally, let's consider a general \emph{auto-regressive policy representation} for $n$ agents under $X=\{x_1,\dots,x_n\}$, a permutation over 1 to $n$, denoting an \emph{execution order} for the agents.
Given any execution order $X$, we can factorize the joint policy $\pi_\theta$ into the form of
\begin{equation}
\label{eq:ar}
    \pi_\theta(\mathbf{a}\mid \mathbf{o}) \approx \prod_{i=1}^n \pi_{\theta^{x_i}}(a^{x_i}\mid o^{x_i},a^{x_1},\dots,a^{x_{i-1}}),
\end{equation}
where the action produced by agent $x_i$ depends on its observation  $o^{x_i}$ and all the actions from its previous agents $x_1,\ldots,x_{i-1}$ under the execution order $X$.
The auto-regressive factorization can represent any joint policy in a centralized MDP.
% Note that the auto-regressive factorization is \emph{equivalent to} the joint policy in the fully observable setting while maintaining a strong expressiveness in Dec-POMDPs.
Moreover, by sequentially generating actions according to $X$, the output dimension of each agent's policy $\pi_{\theta_{x_i}}$ remains unchanged, which results in a minimal policy computation overhead. 

We remark that \emph{auto-regressive policy} in Eq.~(\ref{eq:ar}) is based on a different factorization scheme from the classical decentralized learning paradigm in Eq.~(\ref{eq:pg}).
To sufficiently distinguish these two factorization schemes in our paper, we call the representation in Eq.~(\ref{eq:pg}) \emph{independent policy}.

In theory, we argue that an auto-regressive policy could learn substantially more diverse policies on the $n$-player variant of XOR game, called \emph{permutation game}. Here, we measure the diversity of a policy $\pi$ by the entropy of its trajectories, i.e., $\gH(\pi) = -\E_{\tau \sim \pi}[\log p_\pi(\tau)]$, where $\tau$ denotes a trajectory and $p_{\pi}(\tau)$ is the probability of $\tau$ under $\pi$.

%\fuwei{todo: move the definition}

\begin{definition}[Permutation game] \label{def:xor} An $n$-agent permutation game is a stateless Dec-POMDP. Each agent $i$ has $n$ actions $\gA = \{1, \cdots, n\}$. The reward function is $r(a^1, \cdots, a^n) = \mathbb I\{(a^1, \cdots, a^n) \text{ is a permutation}\}$.

Note that the XOR game is the 2-player permutation game.	
%We also call a two-agent permutation game an XOR game.
\end{definition}

\iffalse
\begin{align}
	\gH(\pi) = -\E_{\tau \sim \pi}[\log p_\pi(\tau)],
\end{align}
where $\tau$ denotes a trajectory and $p_\pi(\tau)$ is the probability of trajectory $\tau$ when executing policy $\pi$. 
\fi

\begin{theorem} For $n$-agent permutation game, the optimal independent policy has entropy $\gH(\pi^\star_{\mathrm{ind}}) = 0$, while the optimal auto-regressive policy could have entropy $\gH(\pi^\star_{\mathrm{auto}}) = \log(n!)$.
\end{theorem}

\begin{proof} The optimal independent policy is given by $\pi^\star_{\mathrm{ind}} = (\pi_1, \cdots, \pi_n)$, where $\pi_1 = \cdots = \pi_{n-1} = 0$ and $\pi_n = 1$. Note that this is a deterministic policy, so $\gH(\pi) = 0$. 
	
	The optimal auto-regressive policy is given by $\pi^\star_{\mathrm{auto}} = (\pi_1, \cdots, \pi_n)$, where $\pi_i(a_i | a_{1}, \cdots, a_{i-1}, s_0) = \frac{1}{n-i}\mathbb{I}\{a_i \ne a_1, a_2, \cdots, a_{i-1}\}$. Note that the joint policy is $\pi = \mathrm{Unif}(\{(a_1, \cdots, a_n) \mid a_i \text{ is a permutation}\})$, which is a uniform distribution over $n!$ permutations, so its entropy is $\gH(\pi) = \log(n!)$. 
\end{proof}

\begin{wrapfigure}{r}{0.4\columnwidth}
\centering\vspace{-28pt}
\subfigure{\vspace{-4pt}\includegraphics[width=0.4\columnwidth]{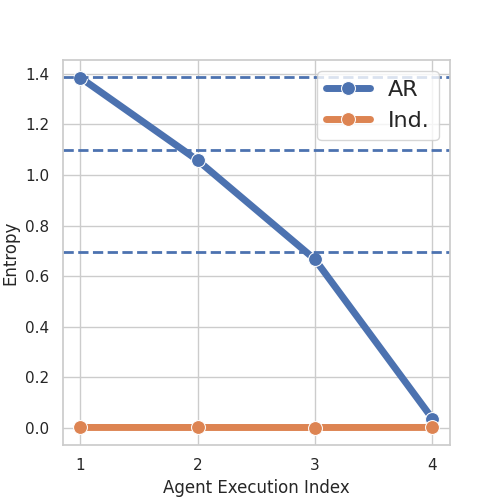}}
\vspace{-4mm}
\caption{\small{Entropy during action selection in the 4-player permutation game. 
% \yw{separate->Ind}
}}
\label{fig:xor-ent}
\vspace{-2mm}
\end{wrapfigure}
We empirically evaluate the effectiveness of auto-regressive policy learning in the 4-player permutation game. We plot the entropy of policies learned by auto-regressive (AR) learning and individual (Ind.) learning in Fig.~\ref{fig:xor-ent}.
The results are averaged over 3 seeds with negligible variance.
%Note that even though the results are averaged over 3 seeds, the variance is almost negligible. 
The AR policy consistently learns a multi-modal behavior while individual learning always converges to a single mode with a zero policy entropy. 
Moreover, we also pick policies trained on a particular trial and illustrate the distribution of joint actions in Fig.~\ref{fig:xor-pg-heatmap}.
Note that there are a total of $4^4=256$ possible joint actions in the 4-player permutation game while only $4!=24$ of them yield a positive reward. We can observe that the AR policy successfully covers all the optimal modes while independent policies only converge to a specific mode. 

\textbf{Remark: }
Auto-regressive policy learning is a minimal approximation of centralized learning, which is beyond the setting of decentralized learning on Dec-POMDP~\cite{oliehoek2008optimal}. 
Most decentralized MARL methods formulated on Dec-POMDP
%Although policy learning under Dec-POMDP has achieved great successes, most algorithms 
are based on independent policy factorization (Eq.~(\ref{eq:vd}, \ref{eq:igm}, \ref{eq:pg})), which has a fundamental limitation --- the joint policy cannot represent \emph{all} the optimal modes. Therefore, we propose Eq.~(\ref{eq:ar}) to overcome this expressiveness limitation via a mild additional assumption (communication of actions) and computation overhead.

\begin{figure}[t]
\centering
\subfigure{\includegraphics[width=\columnwidth]{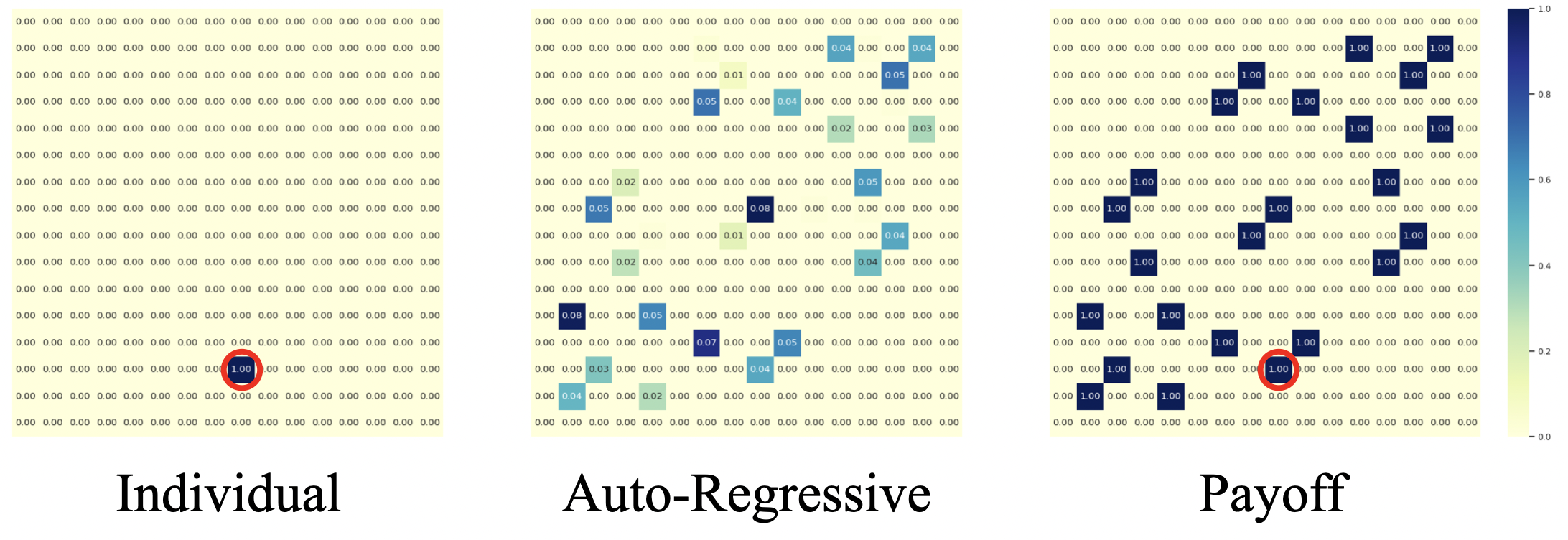}}
\vspace{-3mm}
\caption{Heatmap demonstrating the frequency of every possible joint action during evaluation of 1000 episodes.
X-axis indicates the joint action of the first two agents and Y-axis indicates the joint action of the last two agents, which forms a $16\times16$ matrix. The red circles indicate the single optimal mode discovered by individual PG. Auto-regressive PG can discover all the permutations that solve the game.}
\label{fig:xor-pg-heatmap}
\vspace{-2mm}
\end{figure}

\section{Bridge: a Temporal {XOR} Game}
% \yw{1. introduce the grid world bridge game}

% \yw{3. our suggestions: general form of individual policies and attention-based auto-regressive representation; training paradigm for auto-regressive policies}

% \yw{4. compare the performance of Qmix; VDN; QPlex and individual PG; shared PG with agent embedding; auto-regressive policy}

% \yw{5. ***** do we want to show the exploration power and entropy findings of auto-regressive policy other than its representation capability ??? *****}

We further extend our study to a grid-world Markov game, \emph{Bridge}, as shown in Fig.~\ref{fig:bridge}. In this game, two agents spawn symmetrically at the two corners of the map. Each agent needs to get through the bridge to reach the spawn point of the other agent. At each timestep, each agent receives a penalty which is proportional to the distance between the current position and its goal. 
More environment details can be found in Appendix~\ref{app:env-bridge}.

In this game, each grid may contain only one agent. Swapping positions between two agents within a single timestep is not permitted. % either. 
%In this game, only one agent can occupy a single grid at the same time, and conflict moves, such as neighbor agents crossing over each other, are not allowed. 
Therefore, two agents cannot pass the bridge simultaneously.
As shown in the top row from Fig.~\ref{fig:bridge}, one agent must temporarily leave the bridge and wait until the other one passes by.  %only one agent can pass the bridge at the same time. 
%At the tie state shown in the top-left part in Fig.~\ref{fig:bridge}, each agent can choose to step on the bridge or allow the other agent to pass first. 
By contrast, if both agents perform the same actions and both enter the bridge, it will result in a dead loop, and additional penalties may be incurred for both agents due to time waste. % since one agent has to move back. 
%However, if they choose the same action, time will be wasted and an additional penalty will be given to both agents. 
%In this sense, 
\emph{Bridge} can be interpreted as a temporal version of the XOR game since the two agents need to perform different macro actions, i.e., either wait or move, to achieve the optimal reward. Likewise, there are symmetric optimal strategies. 
%We remark that \emph{Brdige} is a temporal version of the permutation game in the previous section.

\subsection{Agent-Specific Policy Learning on Bridge}
\label{sec:bridge-optimality}
% \yw{First: describe practical suggestions: (1) shared policy with agent ID embedding; (2) individual policy learning. Note that although from the theoretical inspiration, individual policies can be preferred, in practice, fewer parameters will be easier for learning. So we suggest an alternative: use an agent-ID conditioned policy. So we rename this paradigm as \emph{agent-specific policy learning}}

\begin{wraptable}{R}{.46\columnwidth}
  \centering
  \vspace{-1mm}
  \begin{small}
  \begin{tabular}{l|c}
    \toprule
    Method & Reward\\
    \midrule
    QMIX & -1.18(0.70)\\
    QPLEX & -1.48(1.30)\\
    PG-sh & -0.64(0.02)\\
    PG-ID & \textbf{-0.48(0.00)}\\
    PG-Ind & \textbf{-0.48(0.00)}\\
    \midrule
    optimal & \textbf{-0.48}\\
    \bottomrule
  \end{tabular}
  \end{small}
  \vspace{-1mm}
      \caption{Evaluation results in Bridge over 3 seeds.}
      \label{tab:bridge-score}
      \vspace{-3mm}
\end{wraptable}

% \paragraph{Policy Parameterization}
Based on our theoretical analysis in Sec.~\ref{sec:xor}, individual policy learning with unshared parameters (PG-Ind) would be preferred.
However, learning a separate policy for each agent introduces more model parameters and may challenge optimization. We consider an alternative to PG-Ind by learning an \emph{agent-ID-conditioned policy} (PG-ID): parameters are still shared across agents but the observation $o^i$ is concatenated with the one-hot agent ID as an additional policy input feature, which enables the policy to become \emph{agent-specific}.
The effectiveness of such an ID-conditioning technique was also studied by~\citet{yu2021surprising}.

We remark that PG-Ind and PG-ID are two implementation choices to realize an individual policy $\pi_i$ conditioning on agent index $i$. PG-Ind parameterizes each $\pi_i$ as $\pi_{\theta_i}(a^i\mid o^i)$ (conditioning by using different parameters) while PG-ID adopts $\pi_\theta(a^i\mid o^i;i)$ (conditioning by different inputs). Both two choices are valid under the universal approximation theorem~\cite{hornik1989multilayer}.
In practice, $\pi_{\theta_i}(a^i\mid o^i)$ has a stronger conditioning power than $\pi_\theta(a^i\mid o^i;i)$ but may be harder to train due to more parameters and fewer data for each $\pi_{\theta_i}$.
% The choice between PG-ID and PG-Ind may have an impact on algorithm performance in complicated environments, as we will show in Sec.~\ref{sec:smac-fb}.
% We generally suggest PG-ID, but there may not be a 100\% certain answer.

We compare the empirical performances of agent-specific policy learning, including PG-Ind and PG-ID, with shared policy learning (PG-sh) as well as popular VD algorithms, including QMIX and QPLEX, on the 2-player Bridge game. 
% For a fair comparison, we train PG-based methods for 2k timesteps and VD methods for 20k timesteps to ensure VD methods are properly trained. All the algorithms use the same batch size as well. 
%\fuwei{
All the algorithms use the same batch size and are properly trained with sufficient samples. %}
The final evaluation rewards are shown in Table.~\ref{tab:bridge-score}.
Both VD methods empirically fail to solve the game and produce particularly poor final rewards.
Regarding PG methods, since the game is fully symmetric, a shared policy may not solve this game (PG-sh), while all the agent-specific PG methods, i.e., PG-ID and PG-Ind, can consistently learn the optimal behavior with negligible variance. 
Moreover, since the VD methods may not converge in this Bridge game, we can observe that their variances are substantially higher than PG methods. 
%In addition, because VD methods may produce arbitrary undefined behavior in such a game, the variance of evaluation scores is significantly large. Regarding PG methods, since the observation is fully symmetric, a shared policy could not solve this environment, while all of the other PG methods can discover at least one of the optimal modes.

\subsection{Learning Multi-Modal Behavior with Auto-Regressive Policy}
\label{sec:bridge-multi-modal}
% \yw{First: describe empirical issue and the training paradigm; }

% \yw{Second: empirical results; Show that auto-regressive policy learning can learn multi-modal behavior; also ablation study.}

\begin{figure}[t]
\centering
\includegraphics[width=\columnwidth]{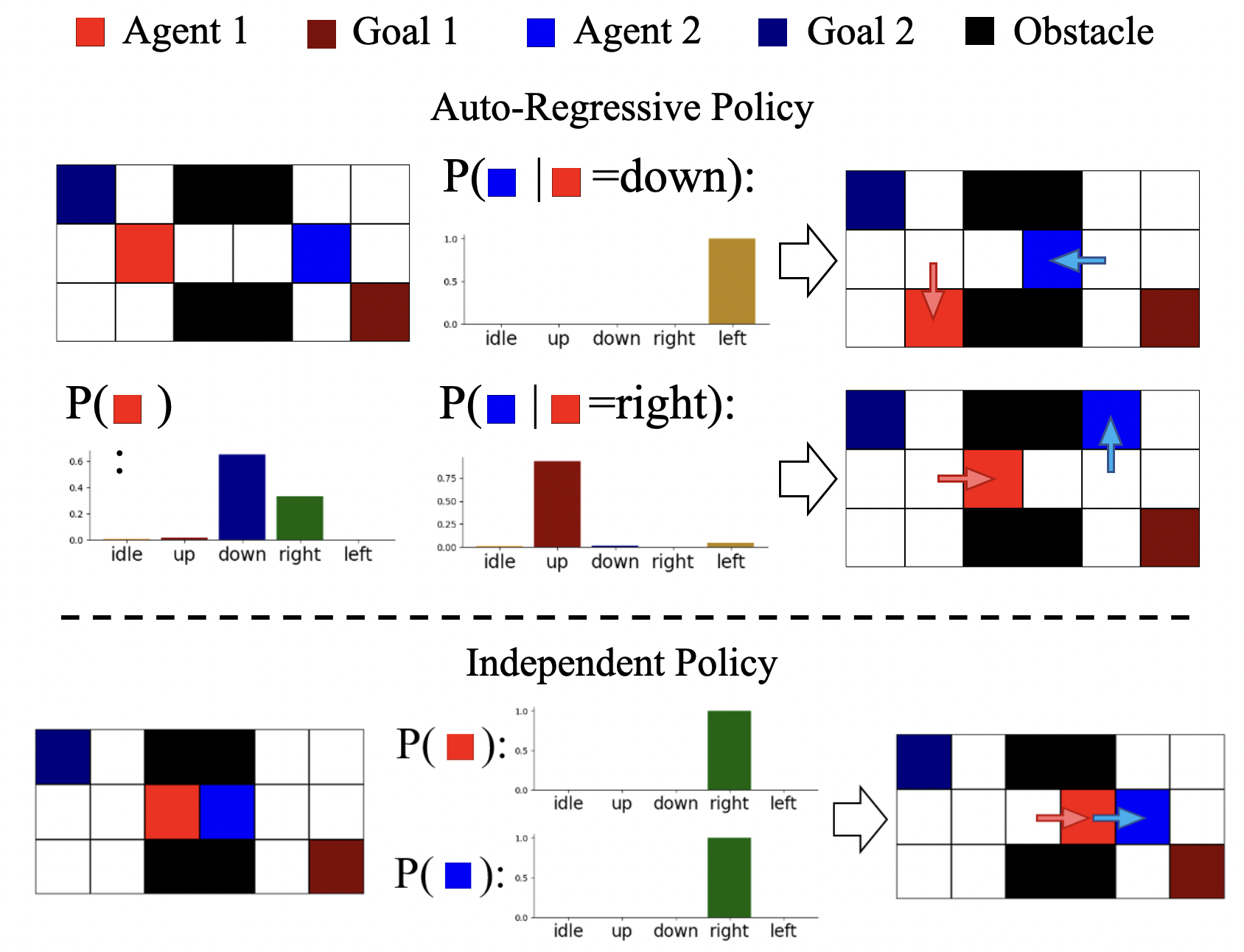}
\caption{(top) Multi-modal behavior by PG-AR in Bridge game. Depending on the action of agent 1 (red), agent 2 (blue) makes different decisions. (bottom) Uni-modal behavior by PG-Ind Agent 1 (red) always passes the bridge first.
% \yw{add texts in the figure: top: Auto-Regressive Policy; bottom: Individual Policy}}
}
\label{fig:bridge}
\end{figure}

We further apply auto-regressive policy to learn multi-modal behaviors in \emph{Bridge} game.
To effectively discover multi-modal behavior
% For effective auto-regressive training
in general Markov games, we propose the following training paradigm. 

\textbf{Attention-based policy: } 
We propose to use an attention-based architecture~\cite{vaswani2017attention}. % as the AR policy representation. 
We convert local observations and input actions into separate embeddings and use a self-attention mechanism to effectively derive a combined input representation.
Such a representation is size-invariant w.r.t. the number of input actions, which further enables parameter sharing. Self-attention also ensures the output properly conditions on the input actions, which helps learn a more coordinated joint policy. %the output indeed conditions on the input actions and accordingly helps learn a more coordinated joint policy. 
%We remark that such a representation also makes mutual information estimation much easier to compute. More details can be found in the appendix.~\yw{appendix: attention-based policy details; explain why it is beneficial for correlation estimation. }

\textbf{Multi-step optimization: }
Given an execution order $X$, each agent $x_i$ will output an action from its own conditioned policy $\pi^{x_i}(a^{x_i}|o^{x_i},a^{x_{1:i-1}})$. Note that since all the output actions are made within the same timestep, every policy $\pi^{x_i}$ has the same return $R$. Therefore, rather than solely optimize the joint policy $\pi$, we can optimize all these $n$ ``partial'' policies by  policy gradient as follows:
\begin{align}
    \nabla J(\mathbf{\pi}) &=\sum_{j=1}^n\mathbb{E}\left[
    -R\cdot \sum_{i=1}^j\nabla\log\mathbf{\pi}(a^{x_i}|o^{x_i},a^{x_{1:i-1}})
    \right].
\end{align}
% \fuwei{
% The above objective is conceptually similar to~\citet{DBLP:journals/corr/abs-2109-11251}. However, actions are served as an additional input to policy, and the "partial" policies are updated simultaneously.
% }

\iffalse
Second, to incorporate the correlation among agents, we propose to directly optimize the joint policy:
\begin{align}
    L(\mathbf{\pi})&=\mathbb{E}_{(s,\mathbf{a})\sim\mathbf{\pi}}\left[
    -\log\mathbf{\pi}(\mathbf{a}|s)R(s,\mathbf{a})
    \right]\\
    &=\mathbb{E}_{(s,\mathbf{a})\sim\mathbf{\pi}}\left[
    -\sum_{i=1}^n\log\mathbf{\pi}(a_i|s,a_{1:i-1})R(s,\mathbf{a})
    \right]
\end{align}
We also want the predecessor agents to provide useful guiding information to successor agents. Therefore, we further optimize all of the \emph{prefix joint policies} $\mathbf{\pi}(a_{1:i}|s)$ for $i\in[n]$. The final loss of the auto-regressive policy is given by

%&=\sum_{j=1}^n \mathbb{E}_{(s,\mathbf{a})\sim\mathbf{\pi}}\left[-R\cdot \nabla\log\mathbf{\pi}\left(a^{(x_{1:j})}|s\right)\right]\\
\fi

\textbf{Randomized execution order:}
Any execution order over $n$ should result in the same factorized distribution for the joint policy $\pi_\theta(\mathbf{a}\mid \mathbf{o})$. Therefore, to prevent the parameterized policy from overfitting a particular order, we randomized the execution order $X$ at each timestep during training.

\begin{table}[t]
    \centering
    \begin{scriptsize}
    \begin{tabular}{cc|ccc}
    \toprule
         PG-Ind&PG-AR& w.o.MO & w.o.RO
         &w.o.Attn.
         \\
         \midrule
         0.01(0.00)&0.74(0.03)&0.02(0.01)& 0.71(0.28)&0.60(0.44)\\ 
         \bottomrule
    \end{tabular}
    \caption{Policy entropy (standard deviation) at the state where both agents are at the bridge ends with ablation studies in Bridge.\label{tab:ablation}}
    \end{scriptsize}
\end{table}
\textbf{Results:} We compared the entropy of learned AR policy (PG-AR) with PG-Ind and perform ablation studies over the techniques, i.e., self-attention (Attn.), multi-step optimization (MO), and randomized execution order (RO). The evaluation results are shown in Table~\ref{tab:ablation}.
We can observe that individual learning (PG-Ind) converges to a specific mode with zero entropy while AR policies (PG-AR) can effectively produce high-entropy strategies.
Note that all three proposed techniques are critical.
% \fuwei{without any of these techniques, PG-AR can degenerate and learn uni-modal behavior like PG-Ind.}
We also illustrate the learned AR policy in the top part of Fig.~\ref{fig:bridge}, where agent 1 (red) captures two modes when entering the bridge, i.e., move right to pass the bridge and move down to wait. Based on the action of agent 1 (red), agent 2 (blue) can output the corresponding optimal action subsequently.

\section{Experiments on Popular MARL Testbeds}

% \yw{1. experiment set up: SMAC and FootBall}

% \yw{2. Suggestion 1: show the final performance of PG with individual policy/agent embedding v.s. other SOTA}

% \yw{3. Suggestion 2: show emergent strategies of auto-regressive solution}

% \yw{4. discussions on limitation and further direction}

%\subsection{Experiment Setup}

Based on the evidence from the Bridge game, we conduct experiments on popular MARL benchmarks, including StarCraft Multi-Agent Challenge (SMAC)~\cite{starcraft} and Google Research Football (GRF)~\cite{kurach2019google}. We use the dense reward setting in both games.
These two environments are intuitively multi-modal since professional football teams or video game players usually have different styles leading to a diverse collection of winning strategies.
In this section, we want to show that 1) in complex multi-modal environments, \emph{value decomposition and policy sharing can be outperformed by agent-specific policy learning methods}, and 2) \emph{auto-regressive policy learning can discover interesting emergent behaviors requiring strong intra-agent coordination}, which are never discovered by existing multi-agent PG methods. 
Our implementation is based on the MAPPO project~\cite{yu2021surprising} with more details in Appendix~\ref{app:imple}. 
%More details can be found in Appendix~\ref{app:env-smac} and \ref{app:env-smac}.

%can discover a multi-modal policy beyond a Markov game. Environment details can be found in Appendix~\ref{app:env-smac} and Appendix~\ref{app:env-smac}.
% Since SMAC has almost been solved by policy gradient algorithms~\cite{yu2021surprising}, we use the more challenging GRF environment to benchmark the performance of state-of-the-art value-decomposition-based and policy gradient methods. Besides, we apply the auto-regressive policy on SMAC to investigate whether it can cover the multi-modal reward landscape in the complex multi-agent environment. More environment details can be found in the appendix.

\begin{table}
\centering
\label{tab:smac}
\begin{scriptsize}
\begin{tabular}{ccccc}
\toprule
Map & PG-ID & PG-sh & PG-Ind & RODE\\
\midrule
      1c3s5z & \textbf{100.0(0.0)} &   97.4(1.0) &     99.1(0.7)  & \textbf{100.0(0.0)}\\
         2s3z & \textbf{100.0(0.7)} &   99.0(0.5) &     99.1(0.9)  & \textbf{100.0(0.0)}\\
    %  3s\_vs\_4z &    100.0(1.3) &   98.6(1.0) &     98.6(1.0)  & \\
     3s\_vs\_5z & \textbf{100.0(0.6)} &   96.7(1.8) &     93.8(1.8) & 78.9(4.2)\\
         3s5z & \textbf{96.9(0.7)} &   95.2(1.5) &     80.4(3.3) & 93.8(2.0)\\
 3s5z\_vs\_3s6z & 84.4(34.0) &   42.3(4.0) &     37.8(5.6) & \textbf{96.8(25.1)}\\
     5m\_vs\_6m & \textbf{89.1(2.5)} &   35.3(2.1) &     44.4(2.9) & 71.1(9.2)\\
     6h\_vs\_8z & \textbf{88.3(3.7)} &   79.9(4.8) &     11.4(2.5) & 78.1(37.0) \\
    %  8m\_vs\_9m & 96.9(0.6) &   53.44(8.4) &     58.02(3.5) \\
  10m\_vs\_11m & \textbf{96.9(4.8)} &   86.5(2.3) &     78.4(2.7) & 95.3(2.2) \\
     corridor & \textbf{100.0(1.2)} &   92.6(2.4) &     82.2(1.8) & 65.6(32.1)\\
         MMM2 & 90.6(2.8) &   \textbf{92.3(1.9)} &     13.0(3.7) & 89.8(6.7)\\
\bottomrule
\end{tabular}
\caption{Median evaluation winning rate (standard deviation) on selected SMAC maps over 6 random seeds. \label{tab:smac}}
\end{scriptsize}
\end{table}

% \begin{table}
% \centering
% \label{tab:smac}
% \begin{scriptsize}
% \begin{tabular}{ccccc}
% \toprule
% Map & PG-ID & PG-sh & RODE\\
% \midrule
%     %   1c3s5z & \textbf{100.0(0.0)} &   97.4(1.0) & \textbf{100.0(0.0)}\\
%         %  2s3z & \textbf{100.0(0.7)} &   99.0(0.5) & \textbf{100.0(0.0)}\\
%     %  3s\_vs\_4z &    100.0(1.3) &   98.6(1.0) &     98.6(1.0)  & \\
%      3s\_vs\_5z & \textbf{100.0(0.6)} &   96.7(1.8)  & 78.9(4.2)\\
%          3s5z & \textbf{96.9(0.7)} &   95.2(1.5)  & 93.8(2.0)\\
%  3s5z\_vs\_3s6z & 84.4(34.0) &   42.3(4.0) & \textbf{96.8(25.1)}\\
%      5m\_vs\_6m & \textbf{89.1(2.5)} &   35.3(2.1)  & 71.1(9.2)\\
%      6h\_vs\_8z & \textbf{88.3(3.7)} &   79.9(4.8)  & 78.1(37.0) \\
%     %  8m\_vs\_9m & 96.9(0.6) &   53.44(8.4) &     58.02(3.5) \\
%   10m\_vs\_11m & \textbf{96.9(4.8)} &   86.5(2.3)  & 95.3(2.2) \\
%      corridor & \textbf{100.0(1.2)} &   92.6(2.4)  & 65.6(32.1)\\
%          MMM2 & 90.6(2.8) &   \textbf{92.3(1.9)}  & 89.8(6.7)\\
% \bottomrule
% \end{tabular}
% \caption{Median evaluation winning rate (standard deviation) on selected SMAC maps over 6 random seeds. \label{tab:smac}}
% \end{scriptsize}
% \end{table}

\begin{table}
\centering
\label{tab:football}
\begin{scriptsize}
\begin{tabular}{ccccc}
\toprule
Scenario & PG-ID & PG-Ind &   CDS(QMIX) &  CDS(QPLEX) \\
\midrule
          3v1 & \textbf{90.7(1.5)} & \textbf{90.2(1.7)} & 73.7(3.4) & 83.6(4.0) \\
     CA(Easy) & 79.5(6.7) & \textbf{92.4(2.6)} & 43.0(5.6) & 40.4(4.8) \\
     CA(Hard) & \textbf{68.2(2.0)} & \textbf{67.8(3.4)} & 35.4(2.9) & 34.8(3.4) \\
      Corner & \textbf{27.3(1.3)} & 21.6(2.4) &  1.8(0.3) & 20.8(1.7) \\
          PS & 43.4(8.8) & 53.3(3.5) & 83.5(4.0) & \textbf{86.8(1.9)} \\
          RPS & 66.6(3.1) & \textbf{78.8(1.5)} & 65.5(7.0) & 75.1(2.4) \\
\bottomrule
\end{tabular}
\caption{Median evaluation winning rate (standard deviation) in GRF academy scenarios over 6 random seeds. (CA=counter attack; PS=pass and shoot; RSP=run, pass and shoot). \label{tab:football}
% \yw{to single column}
}
\end{scriptsize}
\end{table}

% \begin{table}
% \centering
% \label{tab:football}
% \begin{scriptsize}
% \begin{tabular}{ccccc}
% \toprule
% Scenario & PG-ID & PG-Ind  &  CDS(QPLEX) \\
% \midrule
%           3v1 & \textbf{90.7(1.5)} & \textbf{90.2(1.7)}  & 83.6(4.0) \\
%      CA(Easy) & 79.5(6.7) & \textbf{92.4(2.6)}  & 40.4(4.8) \\
%      CA(Hard) & \textbf{68.2(2.0)} & \textbf{67.8(3.4)}  & 34.8(3.4) \\
%       Corner & \textbf{27.3(1.3)} & 21.6(2.4)  & 20.8(1.7) \\
%           PS & 43.4(8.8) & 53.3(3.5)  & \textbf{86.8(1.9)} \\
%           RPS & 66.6(3.1) & \textbf{78.8(1.5)}  & 75.1(2.4) \\
% \bottomrule
% \end{tabular}
% \caption{Median evaluation winning rate (standard deviation) in GRF academy scenarios over 6 random seeds. (CA=counter attack; PS=pass and shoot; RSP=run, pass and shoot). \label{tab:football}
% % \yw{to single column}
% }
% \end{scriptsize}
% \end{table}

% \begin{table}
% \centering
% \label{tab:football}
% \begin{scriptsize}
% \begin{tabular}{ccccc}
% \toprule
% Scenario &       PG-ID &     PG-Ind. &   CDS(QMIX) &  CDS(QPLEX) \\
% \midrule
%           3v1 & \textbf{88.03(1.06)} & 86.32(2.47) & 76.60(3.27) & 69.81(1.20) \\
%      CA(easy) & \textbf{87.76(1.34)} & 81.45(1.71) & 63.28(4.89) & 34.67(1.91) \\
%      CA(hard) & \textbf{77.38(4.81)} & 71.08(6.05) & 53.25(4.17) & 20.21(3.06) \\
%       Corner & 53.53(1.91) & \textbf{64.50(2.05)} &  3.18(0.64) & 14.71(1.03) \\
%           PS & \textbf{94.92(0.68)} & 87.67(2.46) & 94.15(2.54) & 82.05(2.44) \\
%           RPS & \textbf{76.83(1.81)} & 75.30(1.81) & 62.38(4.56) & 48.78(2.04) \\
% \bottomrule
% \end{tabular}
% \end{scriptsize}

\subsection{Learning Policies with Higher Rewards}

% \yw{SMAC: (1) agent ID can be critical and effective in all the maps; (2) individual policy is comparable and can even hurt when the map has too many agents.}

% \yw{Football: (1) fully observable. so agent ID conditioning is essential. (2) individual policy learning further improves the performance of MAPPO (compare with Qmix and CDS). }

We compare agent-specific policy learning, i.e., agent-ID-conditioned policy (PG-ID) and individual policy with unshared parameters (PG-Ind), with state-of-the-art VD-based algorithms, including  RODE~\cite{wang2020rode} for SMAC and CDS~\cite{li2021celebrating} for GRF.
We also include the performance of shared policy learning (PG-sh).

%Since GRF is fully observable and agent ID is by default integrated into the policy input of each agent in the game, we will omit the performance

%the vectorized representation to indicate the current active player and enable efficient learning~\cite{kurach2019google}. Hence, we omit the results of PG-share in the GRF environment. 

Evaluation results are shown in Table~\ref{tab:smac} and Table~\ref{tab:football} for SMAC and GRF respectively. 

In SMAC, PG methods outperform VD methods on 9 out of 10 selected maps as reported by \citet{yu2021surprising}. The agent ID can be critical and agent-conditioned policy (PG-ID) outperforms policies without agent ID input (PG-sh) in almost all the maps except \textit{MMM2}, where the performance of PG-ID and PG-sh is comparable.
For individual policies, it can achieve comparable performances with PG-ID on maps with a small number of agents, while on maps with a large number of agents, using unshared parameters may hurt performance. We believe this is due to the issue of model size.
%, while individual policy can even hurt the performance when there are too many agents on the map. 

In GRF, individual policies (PG-Ind) without parameter sharing achieve comparable or even higher results compared with agent-conditioned policies (PG-ID) in a total of 5 scenarios.
Compared with the state-of-the-art algorithm, CDS, which combines a basic VD method with exploration rewards, PG-based methods, even without any intrinsic rewards, outperform CDS on a total of 5 scenarios except a simple scenario \emph{pass-and-shoot}, which only has two agents. We believe this is due to PG converging towards a poor local optimum due to the deceptive distance-based rewards in GRF, which can be possibly addressed by leveraging more advanced exploration techniques in future work.

\textbf{Practical Suggestion: } Always include agent-specific information in the policy input; if there are not many agents in the game, individual policy learning may be worth trying. 

\subsection{Emergent Behavior by Auto-Regressive Modeling}
\label{sec:smac-fb-multi-modal}

We similarly apply auto-regressive policy learning (PG-AR) in SMAC and GRF. Due to an extremely high degree of freedom and the setting of dense reward, it is non-trivial to visually observe human interpretable multi-modal behaviors. Nevertheless, we still found interesting emergent behaviors discovered by auto-regressive policies: these emergent strategies require particularly strong coordination and are never discovered by other PG variants in our experiments. %, which require particularly strong coordination, These behaviors are never discovered by other PG variants in our experiments.

In SMAC, Fig.~\ref{fig:2m1z} visualizes the learned AR policy on the \emph{2m-vs-1z} map. There are two agents (marines) and one environment-controlled enemy. 
In this strategy, \emph{the two marines keep standing still throughout the game}. They perform attacks alternately while ensuring there is only one attacking marine at each timestep. Since the enemy will by default move towards the agent who attacks it, as a consequence of such an alternative attacking scheme, the enemy turns out to be constantly jittering between the two marines until death without getting any chance to even get close to any of the marines at all. By contrast, marines represented by independent policies will keep moving within the map to ensure a safe distance from the enemy. 
Fig.~\ref{fig:3s5z} visualizes another learned AR policy on the \emph{3s-vs-5z} map where 3 agents are trained to fight against 5 enemies. The agents trained by PG-AR learn to spread out on the map and move towards different corners to keep distance from each other so that every single agent can take charge of just 1 or 2 enemies. By contrast, agents by independent policies will never coordinate to spread out. We illustrate the heatmap of agent positions produced by different policies in Fig.~\ref{fig:3s5z-heatmap}, where a significant visitation difference can be observed. AR policies are more coordinated than independent policies. 

\begin{figure}[t]
\centering
\subfigure{\includegraphics[width=\columnwidth]{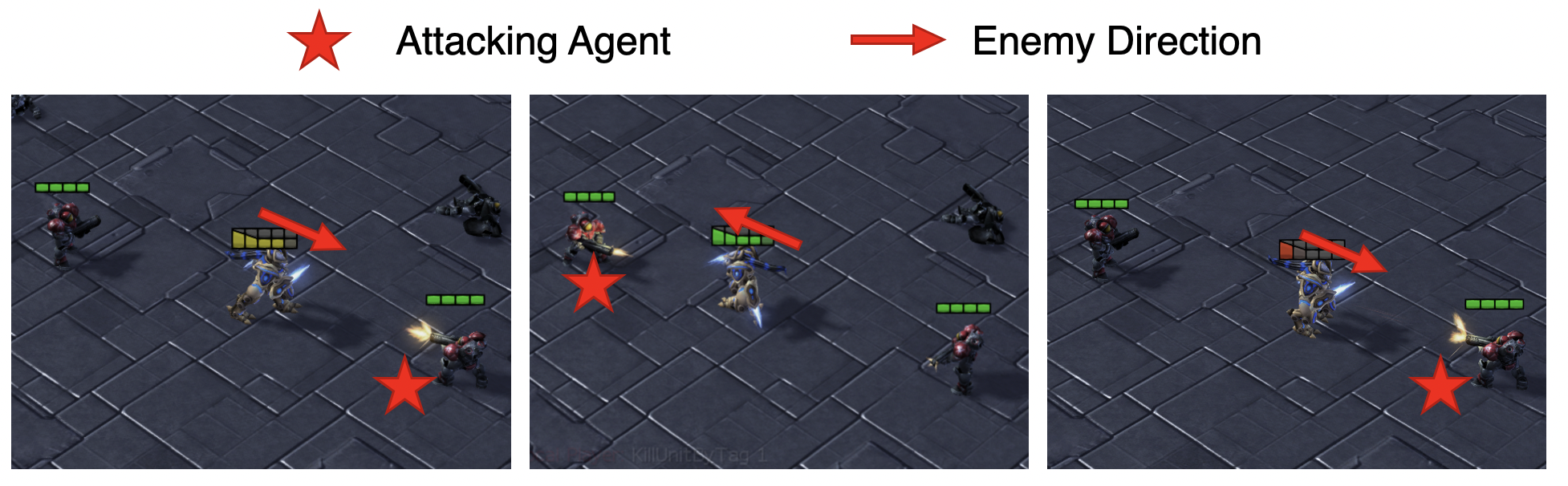}}
\caption{Emergent behavior on the \emph{2m\_vs\_1z} map of SMAC.
% Triangles indicate actions of the controlled marines: red indicates attack and green indicates idle. \yw{add arrows on the enemy}
}
\label{fig:2m1z}
\end{figure}

\begin{figure}[t]
\centering
\subfigure{\includegraphics[width=\columnwidth]{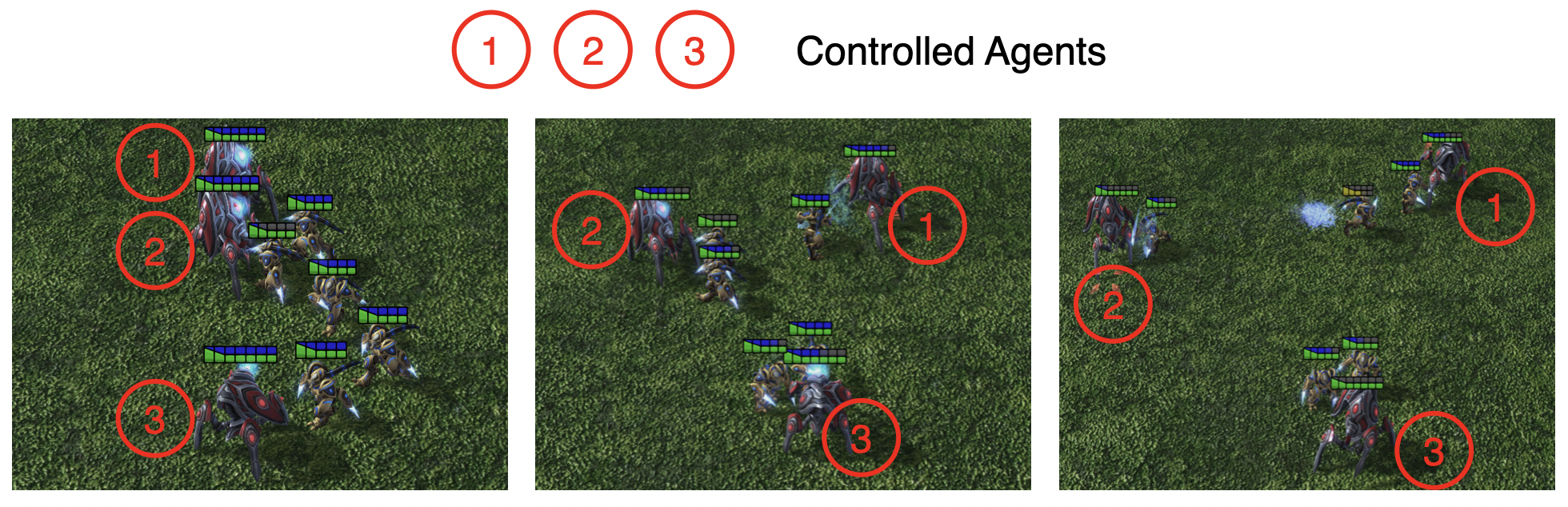}}
\caption{Emergent ``\emph{coordinated-spread-out}'' behavior on the \emph{3s\_vs\_5z} map of SMAC by the AR policy.
}
\label{fig:3s5z}
\end{figure}

\begin{figure}[t]
\centering
\subfigure{\label{fig:3s5z-screenshot}\includegraphics[width=0.38\columnwidth]{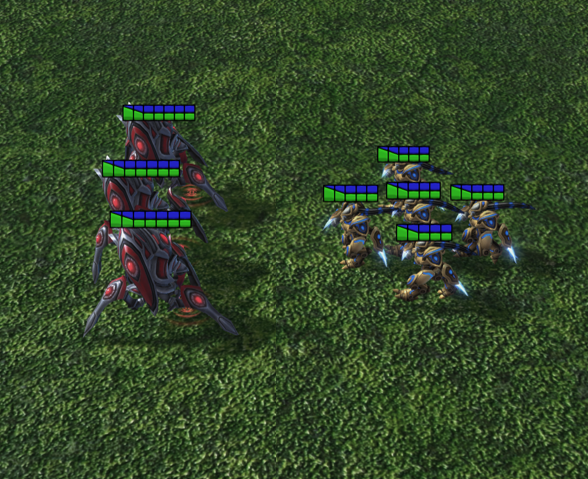}}
\subfigure{\label{fig:3s5z-heatmap}\includegraphics[width=0.6\columnwidth]{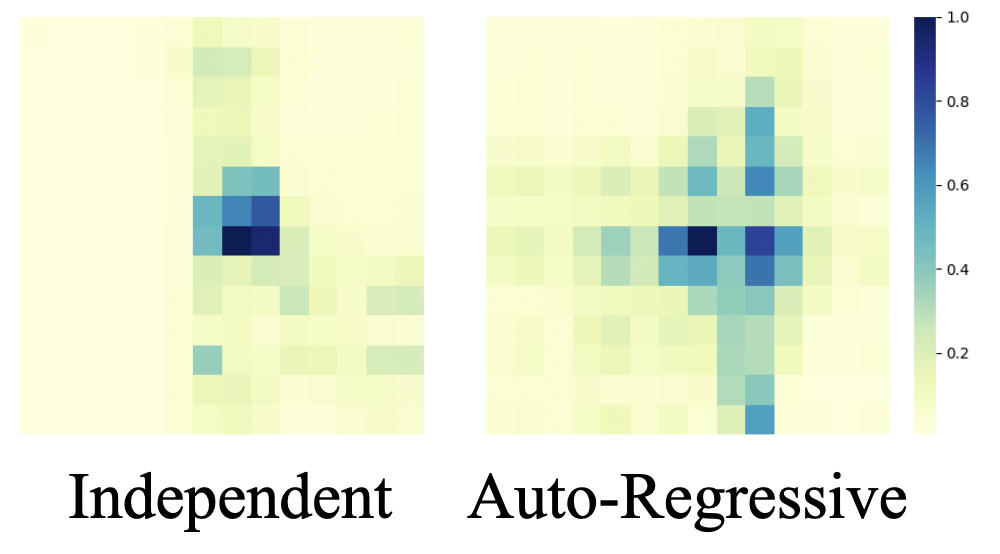}}
\caption{Heatmap of agent positions from the strategies learned by independent policies (middle) and AR policies (right). AR policies are more coordinated and can control the agents to well spread out to different corners of the map. %(left) Screenshot of the \emph{3s\_vs\_5z} scenario. (middle) Agent state coverage of individual policy. (right) Agent state coverage of auto-regressive policy. \yw{change to emergent behavior}
}
\label{fig:3s5z-heatmap}
\end{figure}

\begin{figure}[t]
\centering
\includegraphics[width=\columnwidth]{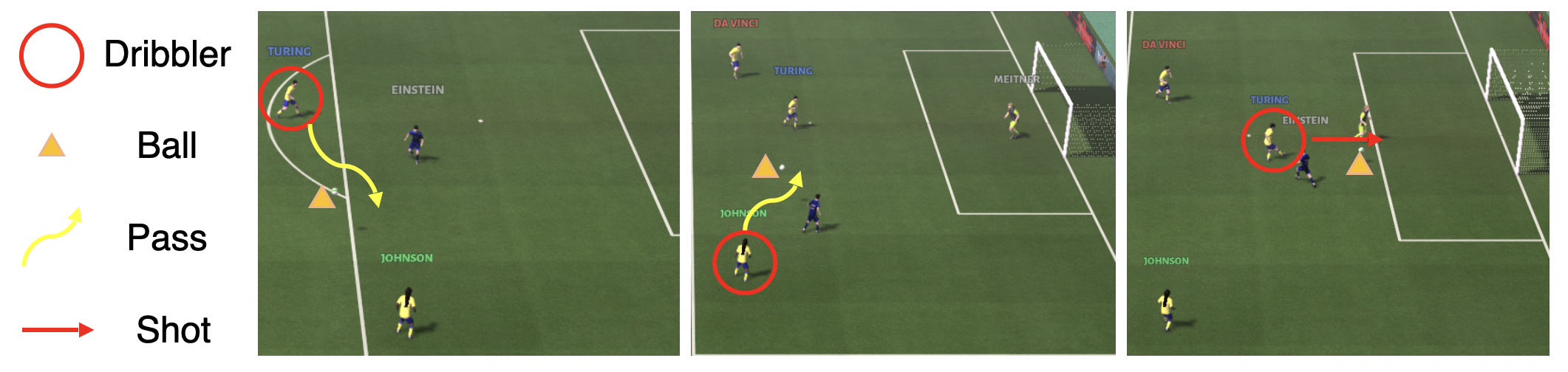}
\caption{Emergent ``Tiki-Taka'' behavior in the GRF \emph{3 vs 1 with keeper} scenario. Red circles indicate the dribbling player.
% Orange triangles indicate the ball position. Yellow and red arrows indicate the pass and shot direction respectively.
}
\label{fig:3v1}
\end{figure}

In GRF, we present the learned strategy by PG-AR for the \emph{3 vs 1 with keeper} scenario in Fig.~\ref{fig:3v1}, where the AR policy learns a neat ``Tiki-Taka'' style behavior: each of the 3 controlled players keeps passing the ball to their teammates, from outside the penalty to the right side and then to the box, before finally scoring a goal.
We remark that only the AR policy can discover such a strategy that involves interactions among all the 3 controlled players.
In the strategy by independent policies, the active agent will directly perform long shots outside the penalty area without performing short passes to teammates. % learned by other PG methods will perform long shoot directly in distance. 

\textbf{Remark:}
We admit that with additional training techniques (Sec.~\ref{sec:bridge-multi-modal}),
auto-regressive policy learning will converge slower than individual policy learning, particularly in games with a lot of agents.
So there is a trade-off between expressiveness capability and sample efficiency.
More experiments on auto-regressive policy learning can be found in Appendix~\ref{app:happo}.
%\fuwei{
%We conduct further investigation in %Appendix~\ref{app:happo}.
%}
% \fuwei{However, by implementing PG-AR in an alternative way, we can augment MARL algorithms and obtain even stronger performance in challenging environments, while learning less interesting cooperative behavior (see Appendix~\ref{??}).}
In general, when optimizing the final reward is not the only goal of a research project, we would suggest adopting auto-regressive modeling for diverse emergent behaviors.
% \fuwei{
% In addition, we remark that an alternative implementation of PG-AR can boost the algorithm performance. See Appendix~\ref{??} for details.
% }

\section{Conclusion}

In this paper, we provide a concrete analysis of the two common practices of MARL algorithms: value decomposition and policy sharing. Theoretical results show that under highly multi-modal scenarios, both two techniques can lead to unsatisfying behaviors, while policy gradient methods can be preferable for learning the optimal solution as well as for learning multi-modal behaviors.  
We propose practical enhancements for implementing effective policy gradient algorithms in general multi-agent Markov games and achieve strong performances in challenging MARL testbeds including StarCraft MultiAgent Challenge and Google Research Football. 
We hope our empirical suggestions can benefit the practitioners while our theoretical analysis could serve as a starting point toward more general and more powerful MARL algorithms. 
%\yw{TODO: future work and remark}

\subsection*{Acknowledgement}
Yi Wu is supported by 2030 Innovation Megaprojects of China (Programme on New Generation Artificial Intelligence) Grant No. 2021AAA0150000. 

%We then propose \emph{agent-specific policy learning} as a general solution to achieve higher rewards in such multi-modal environments, and \emph{auto-regressive policy representation} as a possible modeling approach to discover multi-modal behaviors. We conduct experiments on a variety of domains to verify our claim and finally provide two practical suggestions.

\bibliography{example_paper}
\bibliographystyle{icml2022}

%%%%%%%%%%%%%%%%%%%%%%%%%%%%%%%%%%%%%%%%%%%%%%%%%%%%%%%%%%%%%%%%%%%%%%%%%%%%%%%
%%%%%%%%%%%%%%%%%%%%%%%%%%%%%%%%%%%%%%%%%%%%%%%%%%%%%%%%%%%%%%%%%%%%%%%%%%%%%%%
% APPENDIX
%%%%%%%%%%%%%%%%%%%%%%%%%%%%%%%%%%%%%%%%%%%%%%%%%%%%%%%%%%%%%%%%%%%%%%%%%%%%%%%
%%%%%%%%%%%%%%%%%%%%%%%%%%%%%%%%%%%%%%%%%%%%%%%%%%%%%%%%%%%%%%%%%%%%%%%%%%%%%%%
\newpage
\appendix
\onecolumn

\section{Paper Website}
Please check our project website \url{https://sites.google.com/view/revisiting-marl} for more information, including visualization of learned strategies.

\section{Environment Details}
\subsection{\emph{Bridge} game}
\label{app:env-bridge}

The observation of \emph{Bridge} game is a 6-dim vector representation combined with [self-position, goal position, ally position], and the action space is a Categorical distribution over [idle, moving up, moving down, moving left, moving right]. The observation is processed to be symmetric to both sides of the bridge. After one agent reaches its own goal, it will be marked as ``dead'', and the [ally position] part of the other agent's observation will be masked. 

\subsection{SMAC}
\label{app:env-smac}
We follow to use the SMAC environment and evaluation protocol in MAPPO~\cite{yu2021surprising}.

\subsection{GRF}
\label{app:env-grf}
We use separate dense rewards for all algorithms, i.e., agents obtain independent ``scoring'' and ``checkpoints'' rewards at each timestep. Note that the reward setting is slightly different from that in~\citet{yu2021surprising}, where \textit{shared} dense rewards are used. We use the full action set and the ``simple115v2'' vector representation as the input of both policy and value inputs.

\section{Implementation Details}
\label{app:imple}

\subsection{Fitting the XOR Game}

We use $2\times2$ trainable parameters to represent the local Q-values for each agent, and additional trainable parameters for the mixing network. Specifically, we use 2 trainable parameters as the weight of VDN, a one-layer neural network with 64 hidden units as the hyper-net for QMIX, and 4-head attention with 64 hidden units for QPLEX. All the methods apply stochastic gradient descent with a learning rate of 0.1.

\subsection{Attention-Based Auto-regressive Backbone}
We split the observation into agent-wise slots indicating different observable semantic information. Specifically, slots in the Bridge game contain ``self'', ``goal'', and ``ally''; slots in SMAC contain ``self'', ``move'', ``ally'', and ``enemy''; slots in GRF contain ``self'', ``ball'', ``ally'', and ``enemy''. We embed different slots using different embedding layers, and different entries in the same slot share the same embedding layer, which is similar to the architecture adopted in \citet{baker2019emergent}. One-hot agent actions are embedded with another embedding layer and added onto the corresponding ``ally'' observation embeddings. Then, the embeddings are passed to a self-attention layer and a feed-forward layer, and the first output slot, which corresponds to the observation of the agent itself, is passed to the policy head to output the action distribution. We also mask the unavailable information in the self-attention layer, e.g., ally actions and positions will be masked if it is not visible. All embedding and self-attention layers have a hidden dimension of 64.

\subsection{Hyperparameters and Other Details}
\label{app:hyper-detail}
Hyperparameters of VD methods (except for CDS and RODE) and PG methods are shown in Table~\ref{tab:bridge-ppo-hyper} and Table~\ref{tab:bridge-vd-hyper}. For all the networks and embedding layers, we use 64 hidden units. The backbone of policy, value, and Q network is a 2-hidden-layer MLP for Bridge, with an additional GRU layer for SMAC and GRF. 
We use 4 attention heads for QPLEX and the attention-based backbone of auto-regressive policy. We also add layer norm after each linear layer.
Value normalization is applied to PG methods.
The batch size is 3200 for PG methods in Bridge and SMAC, and 10000 in GRF. The PPO epoch is 5 in Bridge and 15 across all GRF scenarios. PG methods are trained for 50M environment frames on the counterattack-hard and corner scenario, and 25M frames on other scenarios in GRF. In SMAC, we adopt the same PPO epoch and total environment frames as~\citet{yu2021surprising}. For CDS and RODE, we use the hyperparameters and implementation adopted from the original paper.

\begin{table}
    \centering
    \begin{tabular}{c|c}
    \toprule
    Name & value\\
    \midrule
    $\gamma$ & 0.99\\
    GAE~\cite{schulman2015high} $\lambda$& 0.95\\
    PPO clip & 0.2\\
    value clip & 0.2\\
    value loss & huber\\
    huber $\delta$ & 10.0\\
    entropy coefficient & 0.01\\
    optimizer & Adam~\cite{kingma2014adam}\\
    learning rate & 5e-4\\
    gradient norm & 10.0\\
    \bottomrule
    \end{tabular}
    \caption{PG hyperparameters.}
    \label{tab:bridge-ppo-hyper}
\end{table}

\begin{table}[]
    \centering
    \begin{tabular}{c|c}
    \toprule
    Name&Value\\
    \midrule
         $\gamma$&0.99\\
      hard update interval&50\\
      gradient norm&10.0\\
      optimizer & Adam~\cite{kingma2014adam}\\
    learning rate & 5e-4\\
      value loss & MSE\\
      \bottomrule
    \end{tabular}
    \caption{VD hyperparameters.}
    \label{tab:bridge-vd-hyper}
\end{table}

For PG-sh, we use a single actor and a single critic for all the agents. For PG-ID, we additionally concatenate observation with one-hot agent IDs as the input of actor and critic. For PG-ind, we train $n$ separate actors and critics for all agents. For PG-AR, we train a single actor with the attention-based auto-regressive representation and a single critic which is the same as PG-sh for each agent. For auto-regressive policy, we use an entropy coefficient of $0.05$ and omit agent ID.

\section{Additional Results}
\subsection{Q-learning with Value Decomposition in XOR Game}
\label{app:xor-vd-qlearning}
\begin{figure}
    \centering
    \includegraphics{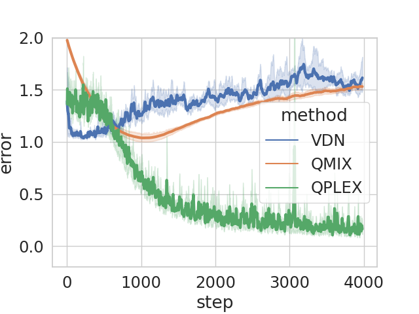}
    \caption{Error in XOR game (Q-learning).}
    \label{fig:vd-error-epsgreedy}
\end{figure}

In addition to directly fitting the payoff matrix of the XOR game (Fig.~\ref{fig:xor-loss}), we also re-run VDN, QMIX, and QPLEX in XOR game with \emph{$\epsilon$-greedy exploration} following the convention of deep Q-learning.  Error curves over 6 random seeds are shown in Fig.~\ref{fig:vd-error-epsgreedy}.
We remark that the generated data using $\epsilon$-greedy are much more \emph{on-policy} than uniform fitting, and it has been reported that stronger on-policiness improves the practical performance of Q-learning-based methods~\cite{laser,fedus2020revisiting}.
Empirically, we observe that in this setting, all of the applied VD methods can learn a global Q function where either $Q_\textrm{tot}(1,2)\approx1$ or $Q_\textrm{tot}(2,1)\approx1$. However, the global Q-values on other entries can be arbitrary, leading to an increasing error of VDN and QMIX across training.
QPLEX is more powerful and the error continuously decreases as training proceeds. However, we note that the error \emph{never} converges to zero, which is consistent with our theoretical analysis (\cref{thm:val}). %Intuitively, when the environment has multiple strategy modes, VD-based Q-learning via Bellman equation will simultaneously fit all the optimal modes.
In addition, we also remark that even though the use of on-policy samples helps VD methods derive an optimal solution in the XOR game, VD methods can completely fail in more complex games, such as the Bridge game (Section~\ref{sec:bridge-optimality}), due to its fundamental representation limitation.
% Consequently, due to the limited expressiveness of VD methods, the optimization process can be undefined and lead to undesired results.

\subsection{Learning Auto-Regressive Policies with Higher Rewards}
\label{app:happo}
\citet{DBLP:journals/corr/abs-2109-11251} proposes Heterogeneous-Agent Proximal Policy Optimization (HAPPO) to solve cooperative MARL problems. In HAPPO, policies are trained sequentially by accumulating importance ratios across agents (see~\citet{DBLP:journals/corr/abs-2109-11251} for more details). However, HAPPO still adopts independent policy factorization in Eq.~(\ref{eq:pg}) and learns an individual policy for each agent, which restricts the expressiveness power. To unlock the full potential, it is then natural to extend HAPPO with our auto-regressive (AR) representation, such that \emph{training and inference are both conducted sequentially.} We name this extension \emph{HAPPO-AR}. By implementing such an algorithm, we can investigate to what extent expressiveness power affects sample efficiency and the performance of the state-of-the-art algorithm. Hopefully, HAPPO-AR can obtain the benefits of both --- it can maintain high sample efficiency while learning interesting emergent behavior that requires strong intra-coordination.

We remark that the focus of HAPPO-AR is different from the focus of learning multi-modal behavior as presented in Sec.~\ref{sec:bridge-multi-modal} and~\ref{sec:smac-fb-multi-modal}. In the main body, we want to show the effectiveness of the AR representation,
% . An AR policy can indeed cover multiple modes in the reward landscape, 
while sample efficiency and algorithm performance are not the most imperative topics.
% The training techniques in Sec.~\ref{sec:bridge-multi-modal} are all designed for learning multi-modal behavior.
By contrast, in this subsection, 
we aim at investigating the effect of the AR representation on algorithm performance. Learning multi-modal behavior is not the ultimate goal. Instead, we hope that the AR representation can aid the algorithm to converge to a better optimum.
Therefore, instead of applying techniques introduced in Sec.~\ref{sec:bridge-multi-modal}, we only employ a minimal modification on HAPPO:
the MLP policy takes one-hot actions of other agents as an additional input. Besides, we align the inference and training order to better utilize auto-regressive learning.

% \begin{table}[]
%     \centering
%     \begin{tabular}{cccc|c}
%     \toprule
%      & PG-ID & PG-Ind &  HAPPO & HAPPO-AR\\
%     \midrule
%           3v1 & 90.7(1.5) & 90.2(1.7) & \textbf{95.1 (5.8)} & \textit{94.9 (4.4)$^\star$} \\
%      CA-E & 79.5(6.7) & \textbf{92.4(2.6)}  & 54.1 (36.5) & \textit{60.5 (35.9)$^\star$} \\
%      CA-H & \textbf{68.2(2.0)} & \textbf{67.8(3.4)}  & 29.2 (43.9)&  16.7 (39.5) \\
%       Corner & 27.3(1.3) & 21.6(2.4)  & \textbf{35.4 (42.5)} &  \textit{39.1 (40.0)$^\star$}\\
%           PS & 43.4(8.8) & 53.3(3.5)  & \textbf{93.5 (6.0)} &  \textit{95.0 (4.5)$^\star$}\\
%           RPS & 66.6(3.1) & 78.8(1.5)  & \textbf{98.7 (4.2)} &  \textit{97.7 (4.2)$^\star$} \\
%     \bottomrule
%     \end{tabular}
%     \caption{Median evaluation winning rate (standard deviation) on GRF academy scenarios over 6 random seeds. Bold numbers are the state-of-the-art results under CTDE. $^\star$ indicates that HAPPO-AR performs at least as well as HAPPO. (CA-E\&H=counter attack easy\&hard; PS=pass and shoot; RSP=run, pass and shoot)}
%     \label{tab:grf-happo}
% \end{table}

\begin{table}[]
    \centering
    \begin{tabular}{cc|c}
    \toprule
      &  HAPPO & HAPPO-AR\\
    \midrule
          3v1  & 95.1 (5.8) & 94.9 (4.4)$^\star$ \\
     CA-E  & 54.1 (36.5) & 60.5 (35.9)$^\star$ \\
     CA-H   & 29.2 (43.9)&  16.7 (39.5) \\
      Corner  & 35.4 (42.5) &  \textit{39.1 (40.0)$^\star$}\\
          PS  & 93.5 (6.0) &  \textit{95.0 (4.5)$^\star$}\\
          RPS  & 98.7 (4.2) &  \textit{97.7 (4.2)$^\star$} \\
    \bottomrule
    \end{tabular}
    \caption{Median evaluation winning rate (standard deviation) on GRF academy scenarios over 6 random seeds. $^\star$ indicates that HAPPO-AR achieves comparable or superior performance.
    We remark that the performance should not be directly compared since HAPPO-AR violates CTDE. (CA-E\&H=counter attack easy\&hard; PS=pass and shoot; RSP=run, pass and shoot)}
    \label{tab:grf-happo}
\end{table}

We evaluate HAPPO and HAPPO-AR on the 6 academy scenarios in Google Research Football following the same evaluation protocol as illustrated in Appendix~\ref{app:hyper-detail} \emph{except that we utilize the shared dense reward} as \citet{yu2021surprising}. Results are presented in
Table~\ref{tab:grf-happo}. HAPPO-AR outperforms or obtains comparable performance with HAPPO in 5 out of 6 academy scenarios.
This result indicates that the AR representation can be combined with the state-of-the-art PG method to improve algorithm performance.
We also visualize the behavior of HAPPO and HAPPO-AR in the 3-vs-1 scenario. While HAPPO learns plain pass-and-shoot, HAPPO-AR can still learn ``Tiki-Taka'' behavior as shown in Sec.~\ref{sec:smac-fb-multi-modal} (see our project website for GIF demonstration). However, in harder scenarios like counterattack-hard and corner,
we do not observe distinguishable behavioral differences.

Finally, we remark that auto-regressive policy learning can be implemented in different ways depending on the purpose. We recommend applying techniques introduced in Sec.~\ref{sec:bridge-multi-modal} if the aim is multi-modal behavior. If the aim is a higher reward, we recommend the implementation introduced in this section.

% HAPPO, ar training, but adopt independent factorization for decentralized execution
% natural to extend HAPPO with ar representation
% how representation power affects algorithm performance
% hope to get the benefits of both: both emergent behavior and high sample efficiency

% difference between multi-modal AR and HAPPO AR
% effectiveness of expressiveness power
% effect on algorithm performance
% training techniques are designed for multi-modal behavior

% > evaluation performance table
% ar may not hurt performance, even improve
% > emergent behavior figure
% ar can learn behavior that requires high-level coordination

\subsection{Evaluation Performance of Multi-Modal Behavior}

\begin{table}
    \centering
    \begin{tabular}{c|c}
    \toprule
         Scenario & Winning Rate\\
         \midrule
         2m\_vs\_1z & 100.0(0.0)\\
         3s\_vs\_5z & 100.0(0.9)\\
         \midrule
         GRF 3v1 & 92.7(3.7)\\
         \bottomrule
    \end{tabular}
    \caption{Median evaluation winning rate of auto-regressive policy in SMAC and GRF.}
    \label{tab:eval-ar}
\end{table}

In Sec.~\ref{sec:smac-fb-multi-modal}, we have shown emergent behavior on the \emph{2m\_vs\_1z} and \emph{3s\_vs\_5z} map in SMAC and on the academy-3-vs-1-with-keeper scenario in GRF. We additionally show the
evaluation performance of these learned policies in Table~\ref{tab:eval-ar}. The results show that PG-AR can successfully learn \emph{winning} multi-modal strategies.

\end{document}